%% file: main.tex
\documentclass{article}

\usepackage{iclr2024_conference,times}

\usepackage[utf8]{inputenc} %
\usepackage[T1]{fontenc}    %
\usepackage{hyperref}       %
\usepackage{url}            %
\usepackage{booktabs}       %
\usepackage{amsfonts}       %
\usepackage{nicefrac}       %
\usepackage{microtype}      %
\usepackage{hyperref}       %
\hypersetup{
    colorlinks=true,
    linkcolor=blue,
    citecolor=blue,
    urlcolor=magenta,
}

\usepackage[toc,page]{appendix}
\usepackage{ulem}
\usepackage{mathtools}    %
\usepackage{bbm}    %
\usepackage{textcomp}    %
\usepackage{caption}
\usepackage{subcaption}
\usepackage{wrapfig}
\usepackage{bm}
\usepackage{multirow}
\usepackage{diagbox}
\usepackage{float}
\usepackage{fixltx2e}

\usepackage{graphicx}
\usepackage{algorithm}
\usepackage{algpseudocode}
\usepackage[colorinlistoftodos]{todonotes}
\usepackage{subcaption}
\usepackage{enumitem}
\usepackage{tablefootnote}
\usepackage{comment}

\newenvironment{proof}[1][Proof]{\paragraph{#1:}}{\hfill$\square$\vspace{4pt}}
\newtheorem{theorem}{Theorem}[section]

\newtheorem{lemma}[theorem]{Lemma}

\newcommand{\source}{{1}}
\newcommand{\transf}{{2}}

\newcommand{\naive}{{\texttt{naive}}}
\newcommand{\oracle}{{\texttt{oracle}}}

\usepackage{amsmath}
\DeclareMathOperator*{\argmin}{arg\,min}

\title{Transferring Learning Trajectories \\ of Neural Networks}

\iclrfinalcopy
\author{Daiki Chijiwa\thanks{Corresponding author: \texttt{daiki.chijiwa@ntt.com}} \\
NTT Computer and Data Science Laboratories, NTT Corporation
}

\begin{document}

\maketitle

\begin{abstract}

\input{abstract.tex}

\end{abstract}

\input{body.tex}

\bibliographystyle{iclr2024_conference}
\bibliography{references}

\newpage
\begin{appendices}
\input{appendix_body.tex}
\end{appendices}

\end{document}

%% file: abstract.tex
Training deep neural networks (DNNs) is computationally expensive, which is problematic especially when performing duplicated or similar training runs in model ensemble or fine-tuning pre-trained models, for example. Once we have trained one DNN on some dataset, we have its learning trajectory (i.e., a sequence of intermediate parameters during training) which may potentially contain useful information for learning the dataset. However, there has been no attempt to utilize such information of a given learning trajectory for another training. In this paper, we formulate the problem of "transferring" a given learning trajectory from one initial parameter to another one (named {\it learning transfer problem}) and derive the first algorithm to approximately solve it by matching gradients successively along the trajectory via permutation symmetry. We empirically show that the transferred parameters achieve non-trivial accuracy before any direct training, and can be trained significantly faster than training from scratch.

%% file: body.tex
\section{Introduction}
Enormous computational cost is a major issue in deep learning, especially in training large-scale neural networks (NNs).
Their highly non-convex objective and high-dimensional parameters make their training difficult and inefficient.
Toward a better understanding of training processes of NNs, their loss landscapes~\citep{hochreiter1997flat,choromanska2015loss} have been actively studied from viewpoints of optimization~\citep{haeffele2017global,li2017convergence,yun2018global} and geometry~\citep{freeman2017topology,simsek2021geometry}.
One of the geometric approaches to loss landscapes is mode connectivity~\citep{garipov2018loss,draxler2018essentially}, which shows the existence of low-loss curves between any two optimal solutions trained with different random initializations or data ordering.
This indicates a surprising connection between seemingly different independent trainings.

Linear mode connectivity (LMC), a special case of mode connectivity, focuses on whether or not two optimal solutions are connected by a low-loss linear path, which is originally studied in relation to neural network pruning~\citep{frankle2020linear}.
It is known that the solutions trained from the same initialization (and data ordering in the early phase) tend to be linearly mode connected~\citep{nagarajan2019uniform,frankle2020linear}, but otherwise they cannot be linearly connected in general.
However, \citet{entezari2021role} observed that even two solutions trained from different random initializations can be linearly connected by an appropriate permutation symmetry.
\citet{ainsworth2023git} developed an efficient method to find such permutations and confirmed the same phenomena with modern NN architectures.
These observations strength the expectation on some sort of similarity between two independent training runs even from different  initializations, via permutation symmetry.

In this paper, motivated by these observations, we make the first attempt to leverage such similarity between independent training processes for efficient training.
In particular, we introduce a novel problem called {\it learning transfer problem}, which aims to reduce training costs for seemingly duplicated training runs on the same dataset, such as model ensemble or knowledge distillation, by transferring a learning trajectory for one initial parameter to another one without actual training.
The problem statement is informally stated as follows:

\vspace{-1.5mm}
\paragraph{Learning transfer problem {\normalfont (informal)}.} {\it Suppose that a source learning trajectory $(\theta_\source^0, \cdots, \theta_\source^T)$ is given for some initial parameter $\theta_\source^0$. Given another initial parameter $\theta_\transf^0$, called target initialization, how can we synthesize the learning trajectory $(\theta_\transf^0,\cdots,\theta_\transf^T)$ for $\theta_\transf^0$ efficiently? }
\vspace{1mm}

To tackle this problem, as illustrated in Figure~\ref{figure:transfer learning trajectory}, we take an approach to transform the source trajectory $(\theta_\source^0, \cdots, \theta_\source^T)$ by an appropriate permutation symmetry $\pi$ as in the previous works of LMC.
In Section~\ref{section:learning transfer}, we formulate the learning transfer problem as a non-linear optimization problem for $\pi$.
We also investigate how much the source trajectory for $\theta_\source^0$ can be transformed close to the target trajectory for $\theta_\transf^0$ by the optimal $\pi$, both theoretically and empirically.
We then derive a theoretically-grounded algorithm to approximately solve it, and also develop practical techniques to reduce its storage and computational cost.
Since our final algorithm requires only several tens of gradient computations and lightweight linear optimization in total, we can transfer a given source trajectory very efficiently compared to training from scratch. 
In Section~\ref{section:experiments}, first we empirically demonstrate that learning trajectories can be successfully transferred between two random or pre-trained initializations, which leads to non-trivial accuracy without direct training (Section~\ref{section:experiments:learning transfer}).
Next we further confirmed that the transferred parameters can indeed accelerate the convergence in their subsequent training (Section~\ref{section:experiments:acceleration}).
Finally, we investigate what properties the transferred parameters inherit from the target initializations (Section~\ref{section:what is being inherited}).
Surprisingly, we observed that, even when the source trajectory is trained from poorly generalizing pre-trained initialization $\theta_\source^0$ and thus inherits poor generalization, the  one transferred to more generalizing initialization $\theta_\transf^0$ inherits better generalization from $\theta_\transf^0$.

In summary, our contributions are as follows: (1) we formulated the brand new problem of learning transfer with theoretical evidence, (2) derived the first algorithm to solve it, (3) empirically showed that transferred parameters can achieve non-trivial accuracy without direct training and accelerate their subsequent training, and (4) investigated the benefit/inheritance from target initializations. Finally, more related works are discussed in Appendix~\ref{appendix:related works}.

\begin{figure}
\centering
    \begin{minipage}{0.38\textwidth}
        \centering
        \includegraphics[width=\textwidth]{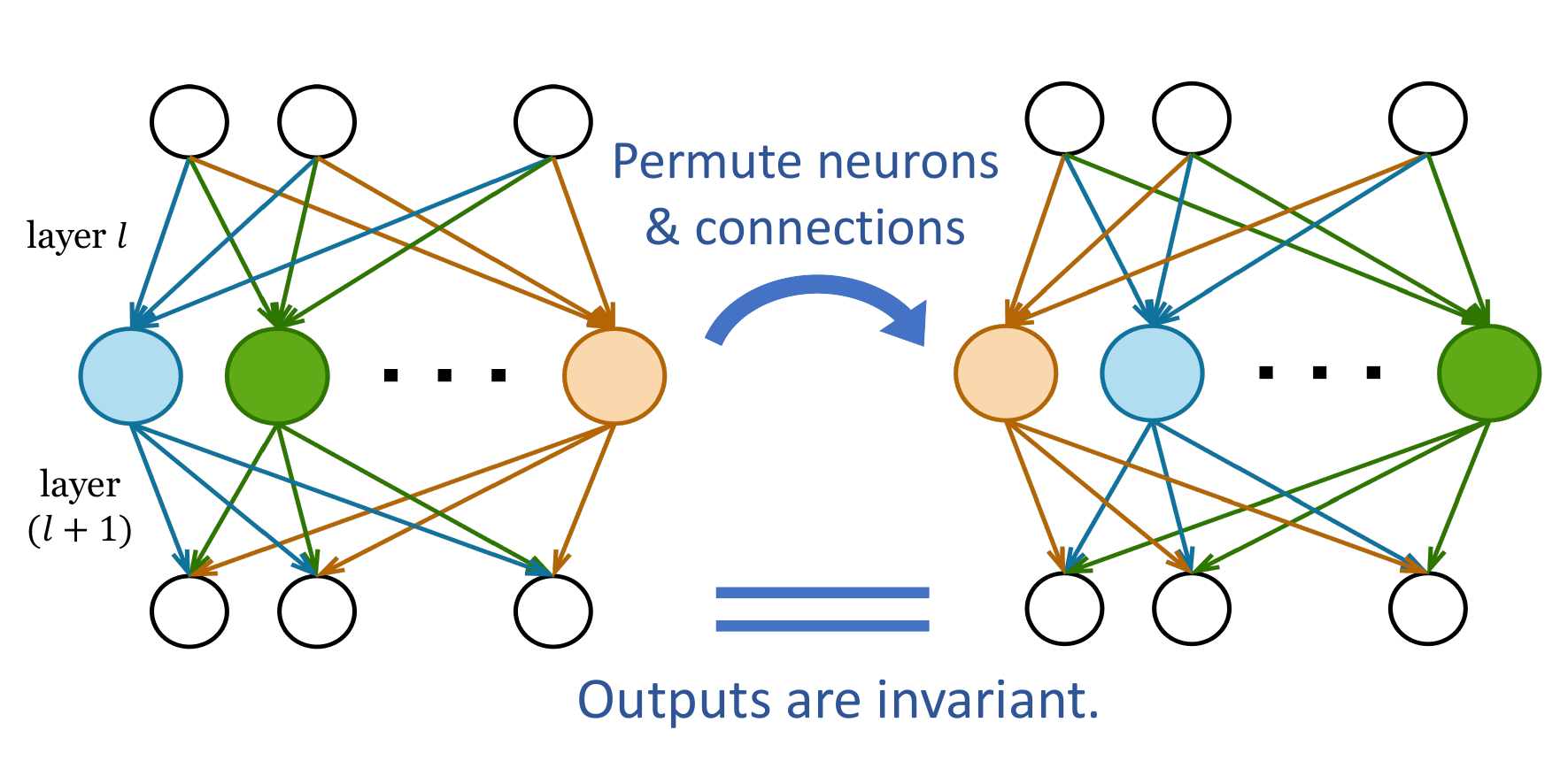}
        \vspace{-5mm}
        \captionof{figure}{Permutation symmetry of neural networks. (Section~\ref{section:permutation symmetry of NNs})}
        \label{figure:permutation symmetry}
    \end{minipage}
    \hspace{0.03\textwidth}
    \begin{minipage}{0.57\textwidth}
        \centering
        \includegraphics[width=\textwidth]{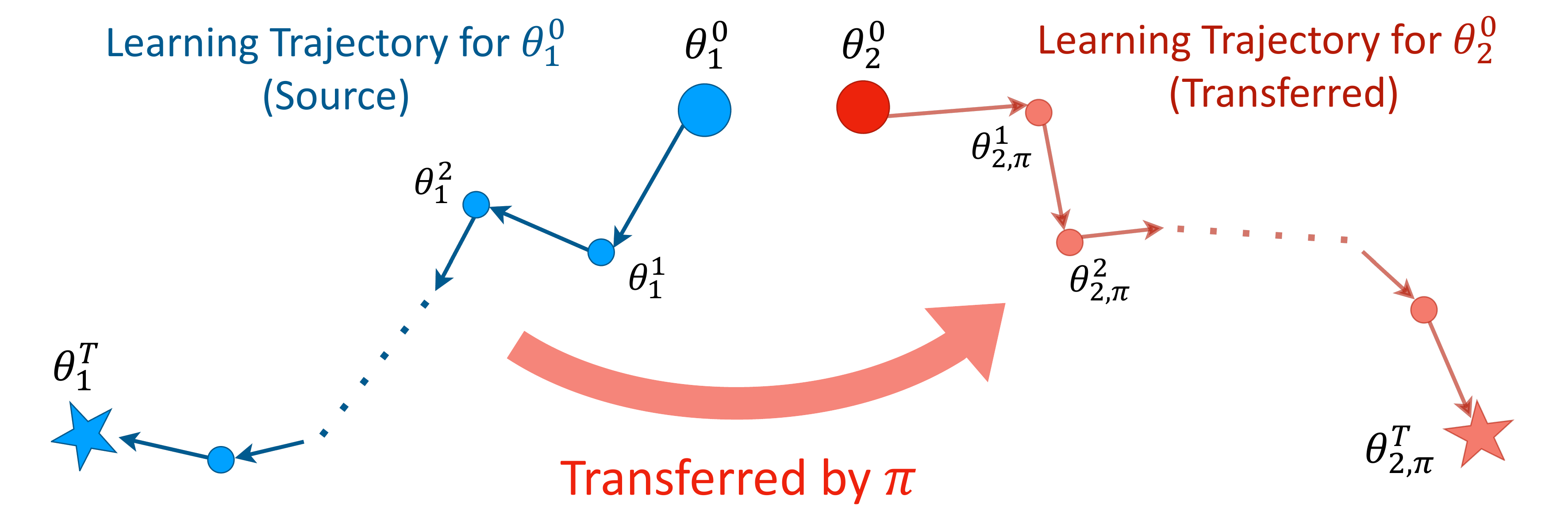}
        \vspace{-5mm}
        \captionof{figure}{Transfer a learning trajectory from one to another NN by a permutation symmetry $\pi$. (Section~\ref{section:learning transfer})}
        \label{figure:transfer learning trajectory}
    \end{minipage}
    \vspace{-3mm}
\end{figure}

\vspace{-1mm}
\section{Background}\label{section:background}

\vspace{-1mm}
\subsection{Neural networks}

Let $L, N \in \mathbb{N}$. Let $f(x;\theta)$ be an $L$-layered neural network (NN) parameterized by $\theta\in\mathbb{R}^N$ with a non-linear activation function $\sigma:\mathbb{R}\to\mathbb{R}$ and intermediate dimensions $(d_0,\cdots,d_L) \in \mathbb{N}^{L+1}$.
Given an input $x\in\mathbb{R}^{d_0}$, the output $f(x;\theta) := x_L \in \mathbb{R}^{d_L}$ is computed inductively as follows:
\begin{equation*}
x_{i} := \begin{cases}
x, & (i=0) \\
\sigma(W_i x_{i-1} + b_i), & (1 \leq i \leq L-1) \\
W_{L} x_{L-1} + b_{L-1}, & (i=L)
\end{cases}
\end{equation*}
where $W_i \in \mathbb{R}^{d_i \times d_{i-1}}, b_i\in\mathbb{R}^{d_i}$ are weight matrices and bias vectors.
Under these notation, the parameter vector $\theta$ is described as $\theta = (W_1, b_1, \cdots, W_{L}, b_L) \in \mathbb{R}^{N}$.

Stochastic gradient descent (SGD) is a widely used approach for training neural networks.
Let $\mathcal{X}$ be the input space $\mathbb{R}^{d_0}$ and $\mathcal{Y}$ be the output space $\mathbb{R}^{d_L}$.
Let $\mathcal{D}$ be a probabilistic distribution over the input-output space $\mathcal{X}\times\mathcal{Y}$, and $\mathcal{L}:\mathcal{Y}\times\mathcal{Y}\to\mathbb{R}$ be a differentiable loss function.
SGD trains a neural network $f(x;\theta)$ by iterating the following steps: (1) Sampling a mini-batch $B = ((x_1,y_1),\cdots,(x_b,y_b)) \sim \mathcal{D}^b$ of size $b \in \mathbb{N}$, (2) computing an estimated gradient $g_B := \frac{1}{b} \sum_{i=1}^b \nabla_\theta \mathcal{L}(f(x_i;\theta), y_i)$ for the mini-batch and (3) updating the model parameter $\theta$ by $\theta - \alpha g_B + \text{(momentum)}$ where $\alpha\in\mathbb{R}_{>0}$ is a fixed or scheduled step size.

\subsection{Permutation symmetry of NNs}\label{section:permutation symmetry of NNs}

For simplicity, we assume that all bias vectors $b_i$ are zero by viewing them as a part of the weight matrices.
Let $\Theta$  be the parameter space $\{\theta = (W_1, \cdots, W_L) \in \mathbb{R}^N\}$ for the $L$-layered neural network $f(x;\theta)$.
Now we introduce a permutation group action on the parameter space $\Theta$.
For $n \in \mathbb{N}$, let $S_n$ denotes the symmetric group over $\{1,\cdots,n\}\subset\mathbb{N}$, which is the set of all bijective mapping $\sigma: \{1,\cdots,n\} \to \{1,\cdots,n\}$.
We define our permutation group $G$ by $G := S_{d_1}\times\cdots\times S_{d_{L-1}}$.
The group action $G \times \Theta \to \Theta$ is defined as follows:
For $\pi=(\sigma_1,\cdots,\sigma_{L-1})\in G$ and $\theta\in\Theta$, the action $\pi\theta$ is defined by
\begin{equation}
    \pi \theta := ( \sigma_1 W_1, \cdots, \sigma_{i} W_i \sigma_{i-1}^{-1}, \cdots, W_L \sigma_{L-1}^{-1}) \in \Theta,
\end{equation}
where each $\sigma_i$ is viewed as the corresponding permutation matrix of size $d_i\times d_i$.
We call this group action the permutation symmetry of $L$-layered neural networks.

Simply speaking, the action $\pi\theta$ just interchanges the axes of the intermediate vector $x_i$ of the neural network $f(x;\theta)$ with the corresponding base change of the weight matrices and bias vectors (Figure~\ref{figure:permutation symmetry}).
Thus we can see that this action does not change the output of the neural network, i.e., $f(x;\pi\theta)=f(x;\theta)$ for every $x\in\mathcal{X},\theta\in\Theta,\pi\in G$.
In other words, the two parameters $\theta$ and $\pi\theta$ can be identified from the functional perspective of neural networks.

\subsection{Parameter alignment by permutation symmetry}\label{section:preliminaries:parameter alignment by permutation}

Previous work by \citet{ainsworth2023git} attempts to merge given two NN models into one model by leveraging their permutation symmetry.
They reduced the merge problem into the parameter alignment problem:
\begin{equation}\label{parameter alignment problem}
    \min_{\pi\in G} \lVert \pi\theta_1 - \theta_2 \rVert_2^2 =  \min_{\pi=(\sigma_1,\cdots,\sigma_{L-1})} \sum_{1\leq i \leq L} \lVert \sigma_i W_i\sigma_{i-1}^{-1} - Z_i \rVert_F^2,
\end{equation}
where $\theta_1 = (W_1,\cdots,W_L)$ and $\theta_2=(Z_1,\cdots,Z_L)$ are the parameters to be merged.
To solve this, they also proposed a coordinate descent algorithm by iteratively solving the following linear optimizations regarding to each $\sigma_i$'s:
\begin{equation}\label{linear sum assignment problem}
    \max_{\sigma_i\in S_i} \langle \sigma_i, Z_i \sigma_{i-1} W_i^\top + Z_{i+1}^\top \sigma_{i+1} W_{i+1} \rangle
\end{equation}
The form of this problem has been well-studied as a linear assignment problem, and we can solve it in a very efficient way~\citep{kuhn1955hungarian}.
Although the coordinate descent algorithm was originally proposed for model merging, we can also use it for other problems involving the parameter alignment problem (eq.~\ref{parameter alignment problem}).

\section{Learning Transfer}\label{section:learning transfer}

In this section, first we formulate the problem of transferring learning trajectories (which we call  {\it learning transfer problem}) as a non-linear optimization problem.
Next, we derive an algorithm to solve it by reducing the non-linear optimization problem to a sequence of linear optimization problems.
Finally, we introduce additional techniques for reducing  the storage and computation cost of the derived algorithm.

\subsection{Problem formulation}\label{section:problem formulation}

Let $f(x;\theta)$ be some NN model with an $N$-dimensional parameter $\theta\in\mathbb{R}^N$.
A sequence of $N$-dimensional parameters $(\theta^0, \cdots, \theta^T)\in\mathbb{R}^{N\times (T+1)}$ is called a {\it learning trajectory} of length $T$ for the neural network $f(x;\theta)$ if the training loss for $\theta^t$ decreases as $t$ increases, from the initial parameter $\theta^0$ to the convergent one $\theta^T$.
Note that \textit{we do not specify what $t$ represents a priori; it could be iteration, epoch or any notion of training timesteps}.\footnote{In our experiments, due to practical constraints, we mainly consider cases where $t$ represents training epoch or linear interpolating step as explained in Section~\ref{section:additional techniques}, which leads to the length $T$ being smaller than a hundred.}
Now we can state our main problem:
\vspace{-1mm}
\paragraph{Learning transfer problem {\normalfont (informal)}.} Suppose that a learning trajectory $(\theta_\source^0, \cdots, \theta_\source^T)$ is given for an initial parameter $\theta_\source^0$, which we call a {\it source trajectory}. Given another initial parameter $\theta_\transf^0$ "similar" to $\theta_\source^0$ in some sense, how can we synthesize the learning trajectory $(\theta_\transf^0,\cdots,\theta_\transf^T)$, which we call a {\it transferred trajectory}, for the given initialization $\theta_\transf^0$? (Figure~\ref{figure:transfer learning trajectory})
\vspace{1mm}

To convert this informal problem into a computable one, we need to define the notion of "similarity" between two initial parameters.
As a first step, in this paper, we consider two initializations are "similar" if two learning trajectories starting from them are indistinguishable up to permutation symmetry of neural networks (Section~\ref{section:permutation symmetry of NNs}).
In other words, for the two learning trajectories $(\theta_\source^0,\cdots,\theta_\source^T)$ and $(\theta_\transf^0,\cdots,\theta_\transf^T)$, we consider the following assumption:
\vspace{-1mm}
\paragraph{Assumption (P).} There exists a permutation $\pi$ satisfying $\pi (\theta_\source^{t} - \theta_\source^{t-1}) \approx \theta_\transf^{t} - \theta_\transf^{t-1}$ for  $t=1,\cdots, T$, where the transformation $\pi (\theta_\source^{t}-\theta_\source^{t-1})$ is as defined in Section~\ref{section:permutation symmetry of NNs}.
\vspace{1mm}

Under this assumption, if we know the permutation $\pi$ providing the equivalence between the source and transferred trajectories, we can recover the latter one $(\theta_\transf^0,\cdots,\theta_\transf^T)$ from the former one $(\theta_\source^0, \cdots,\theta_\source^T)$ and the permutation $\pi$, by setting $\theta_\transf^{t} := \theta_\transf^{t-1} + \pi(\theta_\source^{t}-\theta_\source^{t-1})$ inductively on $t$ (Figure~\ref{figure:transfer learning trajectory}). Therefore, the learning-trajectory problem can be reduced to estimating the permutation $\pi$ from the source trajectory $(\theta_\source^0,\cdots,\theta_\source^T)$ and the given initialization $\theta_\transf^0$.

Naively, to estimate the  permutation $\pi$, we want to consider the following optimization problem:
\begin{equation}\label{naive optimization problem}
    \min_{\pi} \sum_{t=1}^T \left\lVert \pi(\theta_1^t - \theta_1^{t-1}) - (\theta_2^t - \theta_2^{t-1}) \right\rVert_2^2
\end{equation}
However, this problem is ill-defined in our setting since each $\theta_\transf^t$ is not available for $1\leq t \leq T$ in advance. Even if we defined $\theta_\transf^t := \theta_\transf^{t-1} + \pi (\theta_\source^t - \theta_\source^{t-1})$ in the equation~(\ref{naive optimization problem}) as discussed above, the optimization problem became trivial since any permutation $\pi$ makes the $L^2$ norm to be zero. 

Thus we need to fix the optimization problem (eq.~\ref{naive optimization problem}) not to involve unavailable terms. We notice that the difference $\theta_\transf^{t} - \theta_\transf^{t-1}$ can be roughly approximated by a negative gradient at $\theta_\transf^{t-1}$ averaged over a mini-batch if the trajectory is enough fine-grained.
Therefore, we can consider the approximated version of equation~(\ref{naive optimization problem}) as follows:
\begin{equation}\label{gradient matching problem}
    \mathcal{P}_T: \ \  \min_{\pi} \sum_{t=0}^{T-1} \left\lVert \pi \nabla_{\theta_\source^t} \mathcal{L} - \nabla_{\theta_{\transf,\pi}^t} \mathcal{L} \right\rVert_2^2, \text{  where } \theta_{\transf,\pi}^t := \theta_{\transf,\pi}^{t-1} + \pi (\theta_\source^t - \theta_\source^{t-1}).
\end{equation}
In contrast to the equation~(\ref{naive optimization problem}), this optimization problem is well-defined even in our setting because each $\theta_{\transf,\pi}^{t}$ is defined by using the previous parameter $\theta_{\transf,\pi}^{t-1}$ inductively.

\subsection{When and how does Assumption (P) holds?}\label{section:validating assumption of learning transfer problem}

Before delving into the optimization problem $\mathcal{P}_T$, here we briefly investigate when and how our Assumption (P) holds. To measure the similarity of two vectors $\pi(\theta_\source^t - \theta_\source^{t-1})$ and $\theta_\transf^t - \theta_\transf^{t-1}$, we mainly use  a variant of normalized distance $\lVert v_1 - v_2 \rVert / \sqrt{\lVert v_1 \rVert \lVert v_2 \rVert}$ for two vectors $v_1, v_2 \in \mathbb{R}^n$, which can also be considered as cosine distance when $\lVert v_1\rVert \approx \lVert v_2 \rVert$.

First of all, we study the case of $2$-layered ReLU neural networks $f_{w,v}(x) := \sum_{i=1}^N v_i \sigma(\sum_{j=1}^d w_{ij} x_j)$ with $N$ hidden neurons, where $\sigma(z) := \max(z, 0)$ is ReLU activation.
Let $\theta_\source^0=((w_{ij})_{ij}, (v_i)_i)$ and $\theta_\transf^0=((w'_{ij})_{ij}, (v'_i)_i)$ be initialized by Kaiming uniform initialization~\citep{he2015delving}, i.e., $w_{ij},w'_{ij}\sim U([-\frac{1}{\sqrt{d}},\frac{1}{\sqrt{d}}])$ and $v_{i},v'_{i}\sim U([-\frac{1}{\sqrt{N}},\frac{1}{\sqrt{N}}])$. Then it is theoretically shown that Assumption (P) holds at initialization with high probability when the hidden dimension $N$ is sufficiently large:
\begin{theorem}[See Theorem~\ref{appendix:theorem:validating assumption at initialization} for details]
Given two pairs of randomly initialized parameters $(w, v)$ and $(w', v')$, with high probability, there exists a permutation symmetry $\pi \in S_N$ such that the normalized distance between the expected gradients $\mathbb{E}_{(x,y)} [\nabla_{w,v} \mathcal{L}]$ and $\mathbb{E}_{(x,y)} [\nabla_{w'',v''} \mathcal{L}]$, where $(w'', v'')$ is the permuted one of $(w', v')$ by $\pi$, can be arbitrarily small when $N$ is sufficiently large.
\end{theorem}
Although this result only validates Assumption (P) at initialization, we can naturally expect that same similarity holds for more iterations ($t>0$) due to (almost everywhere) smoothness of neural networks with respect to parameters.
Figure~(\ref{figure:assumption:2-mlp on mnist}) empirically validates this expectation.

Also, we empirically investigate the case of modern network architectures in Figure~(\ref{figure:assumption:conv8 on cifar}), (\ref{figure:assumption:resnet18 on imagenet}).
The results show that the similarity between learning trajectories heavily depends on network architecture and type of initialization (i.e., random or pre-trained) rather than datasets.
Unfortunately, we can see that, in such modern network architectures, Assumption (P) holds weakly or may be broken after hundreds of iterations.
However, in this paper, Assumption (P) still plays an important role for deriving our algorithm to solve the main problem $\mathcal{P}_T$, which empirically works well even with these architectures as we will see in Section~\ref{section:experiments}.
It remains for future work to make the assumption more appropriate for these modern architectures.

\begin{figure}
    \centering
    \begin{subfigure}[t]{0.32\textwidth}
        \centering
        \includegraphics[width=\textwidth]{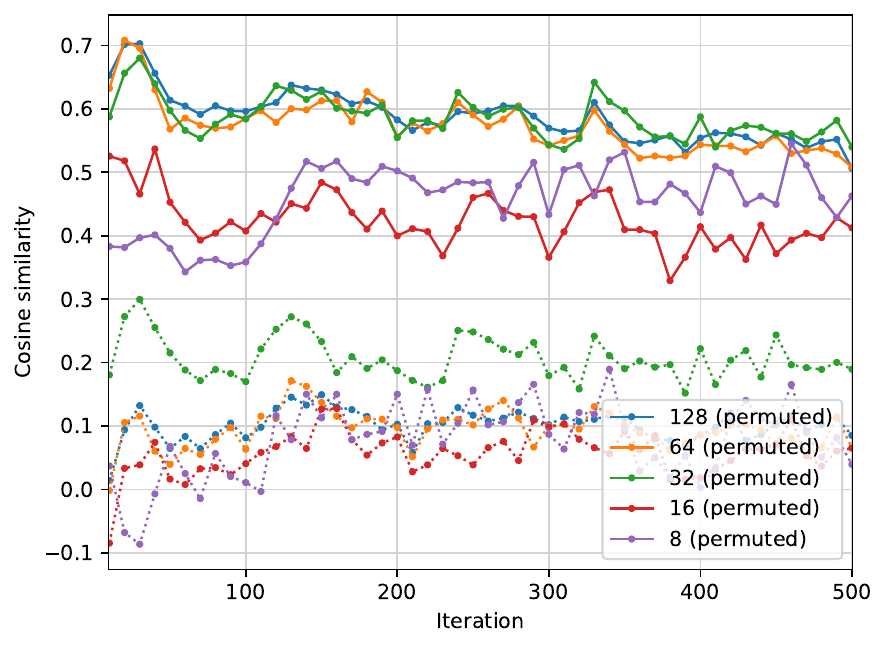}
        \caption{2-MLP on MNIST.}\label{figure:assumption:2-mlp on mnist}
    \end{subfigure}
    \hfill
    \begin{subfigure}[t]{0.32\textwidth}
        \centering
        \includegraphics[width=\textwidth]{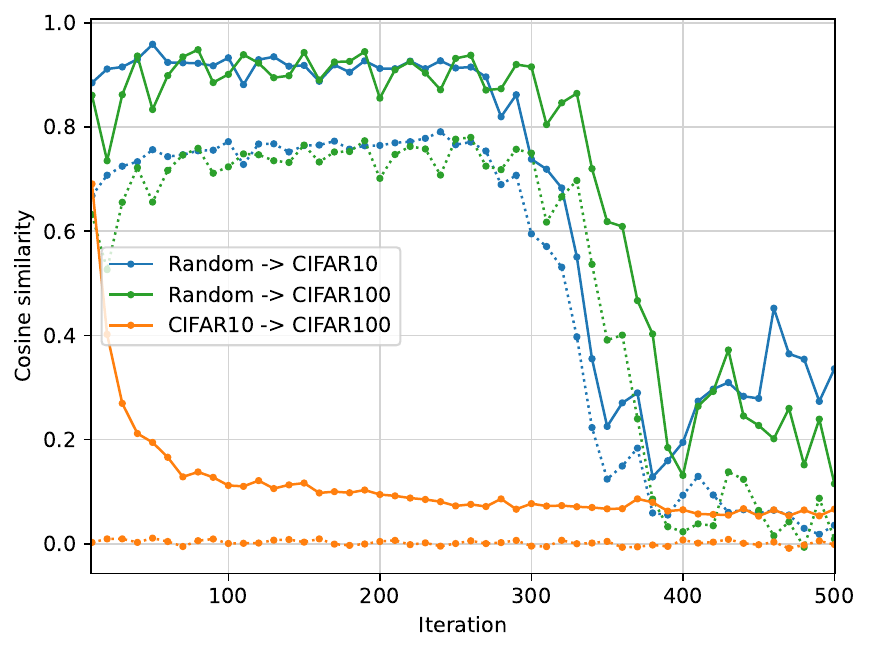}
        \caption{Conv8 on CIFAR-10/100.}\label{figure:assumption:conv8 on cifar}
    \end{subfigure}
    \hfill
    \begin{subfigure}[t]{0.32\textwidth}
        \centering
        \includegraphics[width=\textwidth]{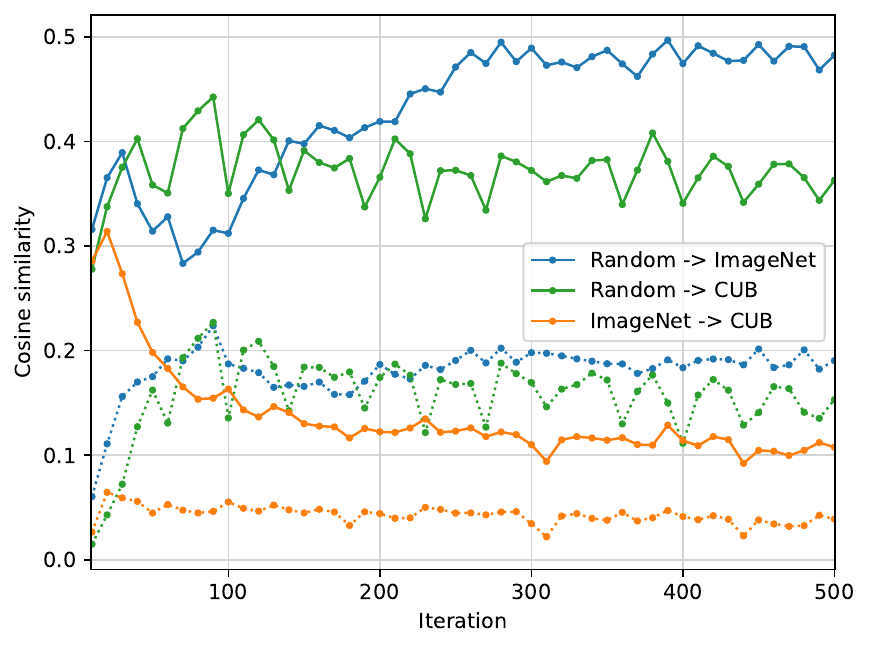}
        \caption{ResNet18 on ImageNet/CUB.}\label{figure:assumption:resnet18 on imagenet}
    \end{subfigure}
    \caption{ We evaluate cosine similarities between $\pi(\theta_\source^t - \theta_\source^{t-1})$ and $\theta_\transf^t - \theta_\transf^{t-1}$, where $\theta_\source^t$ or $\theta_\transf^t$ represents a parameter trained for $10t$ iterations, averaged over timesteps $t$. Solid lines are plotted for the solution $\pi$ of equation~(\ref{naive optimization problem}), and dotted lines are for $\pi$ being the identity transformation (i.e., the case not permuted) as a baseline. Fig.(\ref{sub@figure:assumption:2-mlp on mnist}): Experiments using 2-layered MLP with various hidden dimensions. Over-parameterization leads to higher cosine similarities of learning trajectories.  Fig.(\ref{sub@figure:assumption:conv8 on cifar}, \ref{sub@figure:assumption:resnet18 on imagenet}): The label $X\to Y$ means that the initialization is random or  pre-trained on $X$, and then trained on $Y$. The cosine similarities depend on model architectures and initialization-type rather than datasets. }\label{figure:empirical validation of assumption (P)}
    \vspace{-3mm}
\end{figure}

\subsection{Algorithm: gradient matching along trajectory}\label{section:algorithm}

Now our goal is to solve the optimization problem $\mathcal{P}_T$ (eq.~\ref{gradient matching problem}).
However, the problem $\mathcal{P}_T$ seems hard to solve directly because the variable $\pi$ appears non-linearly in the second term $\nabla_{\theta^t_{\transf, \pi}}\mathcal{L}$.
To avoid the non-linearity, we introduce a sequence of linear sub-problems $\{\mathcal{P}'_{s}\}_{1\leq s \leq T}$ whose solution converges to the solution for $\mathcal{P}_T$. For each $s\in \{1,\cdots, T\}$, we consider the following problem:
\begin{equation}\label{eq:linear sub-problems}
    \mathcal{P}'_s: \ \  \min_{\pi_s} \sum_{t=0}^{s-1} \left\lVert 
    \pi_{s} \nabla_{\theta_\source^t}\mathcal{L} - \nabla_{\theta_{\transf, \pi_{s-1}}^t}\mathcal{L}
    \right\rVert_2^2
\end{equation}
Since the second term in $\mathcal{P}'_s$ uses the solution $\pi_{s-1}$ for the previous sub-problem $\mathcal{P}'_{s-1}$, the unknown variable $\pi_s$ appears only in the first term $\pi_s\nabla_{\theta^1_\source} \mathcal{L}$ in a linear way.
Moreover, the following lemma implies that the final solution $\pi_T$ from the sequence $\{\mathcal{P}'_s\}_{1\leq s \leq T}$ approximates the solution for the original problem $\mathcal{P}_T$:
\begin{lemma}\label{lemma:consistency of subproblems}
Under some regularity assumption, we have $\theta^t_{2,\pi_{s}} \approx \theta^t_{2,\pi_{s'}}$ for $0\leq t \leq s < s'$.
\end{lemma}
The proof will be given in Appendix.
Indeed, by this approximation, we find out that the solution $\pi_T$ for the last sub-problem $\mathcal{P}'_T$ minimizes
\begin{eqnarray}
    \sum_{t=0}^{T-1} \left\lVert 
    \pi_{T} \nabla_{\theta_\source^t}\mathcal{L} - \nabla_{\theta_{\transf, \pi_{T-1}}^t}\mathcal{L}
    \right\rVert_2^2 
    \approx 
    \sum_{t=0}^{T-1} \left\lVert 
    \pi_{T} \nabla_{\theta_\source^t}\mathcal{L} - \nabla_{\theta_{\transf, \pi_{T}}^t}\mathcal{L}
    \right\rVert_2^2,
\end{eqnarray}
where the right-hand side is nothing but the objective of the original problem $\mathcal{P}_T$.

Algorithm~\ref{algorithm:gradient matching along trajectory} gives a step-by-step procedure to obtain the transferred learning trajectory $(\theta_\transf^1, \cdots, \theta_\transf^T)$ by solving the sub-problems $\{\mathcal{P}'_s\}_{0\leq s \leq T}$ sequentially.
In lines 2-6, it computes an average of gradients $\nabla_{\theta}\mathcal{L}$ over a single mini-batch for each $\theta = \theta^{t-1}_\source, \theta^{t-1}_\transf, (1\leq t \leq s)$, which is required in the $s$-th sub-problem $\mathcal{P}'_s$ (eq.~\ref{eq:linear sub-problems}).
In line 7, the $s$-th permutation $\pi_s$ is obtained as a solution of the sub-problem $\mathcal{P}'_s$, which can be solved as a linear optimization (eq.~\ref{linear sum assignment problem}) using the coordinate descent algorithm proposed in \citet{ainsworth2023git}.
Then we update the transferred parameter $\theta_\transf^t$ for $t=1,\cdots,s$ in line 8.

\begin{figure}[htbp]
    \centering
    \begin{minipage}[t]{0.48\linewidth}
        \begin{algorithm}[H]
            \small
            \caption{Gradient Matching along Trajectory ({\bf GMT})}\label{algorithm:gradient matching along trajectory}
            \begin{algorithmic}[1]
                \Require $(\theta_\source^{0},\cdots,\theta_\source^{T}) \in \mathbb{R}^{n\times (T+1)}$, $\theta_\transf^{0}\in\mathbb{R}^n$
                \For{$s=1,\cdots,T$}
                    \For{$t=1,\cdots,s$}
                        \State Sample $(x_1,y_1),\cdots,(x_b,y_b)\sim\mathcal{D}$.
                        \State $g_1^{t} \gets \frac{1}{b} \sum_{i=1}^{b} \nabla_{\theta_1^{t-1}} \mathcal{L}(f(x_i;\theta_1^{t-1}),y_i)$
                        \State $g_2^{t} \gets \frac{1}{b} \sum_{i=1}^{b} \nabla_{\theta_2^{t-1}} \mathcal{L}(f(x_i;\theta_2^{t-1}),y_i)$ 
                    \EndFor
                    \State $\pi_s \gets  \argmin_{\pi} \sum_{t=1}^s \lVert g_2^{t} - \pi g_1^{t} \rVert_2^2$
                    \For{$t=1,\cdots,s$}
                        \State $\theta_2^{t} \gets \theta_2^{t-1} + \pi_s(\theta_1^{t} - \theta_1^{t-1})$
                    \EndFor
                \EndFor
                \State return $(\theta_2^{1},\cdots,\theta_2^{s})$
                \end{algorithmic}
        \end{algorithm}
    \end{minipage}
    \hfill
    \begin{minipage}[t]{0.48\linewidth}
        \begin{algorithm}[H]
            \small
            \begin{algorithmic}[1]
                \Require $(\theta_\source^{0},\cdots,\theta_\source^{T}) \in \mathbb{R}^{n\times (T+1)}$, $\theta_\transf^{0}\in\mathbb{R}^n$
                \For{$s=1,\cdots,T$}
                    \State Sample $(x_1,y_1),\cdots,(x_b,y_b)\sim\mathcal{D}$.
                    \State $g_1^{t} \gets \frac{1}{b} \sum_{i=1}^{b} \nabla_{\theta_1^{s-1}} \mathcal{L}(f(x_i;\theta_1^{s-1}),y_i)$
                    \State $g_2^{t} \gets \frac{1}{b} \sum_{i=1}^{b} \nabla_{\theta_2^{s-1}} \mathcal{L}(f(x_i;\theta_2^{s-1}),y_i)$ 
                    \State $\pi_s \gets  \argmin_{\pi} \sum_{t=1}^s \lVert g_2^{t} - \pi g_1^{t} \rVert_2^2$
                    \For{$t=1,\cdots,s$}
                        \State $\theta_2^{t} \gets \theta_2^{t-1} + \pi_s(\theta_1^{t} - \theta_1^{t-1})$
                    \EndFor
                \EndFor
                \State return $(\theta_2^{1},\cdots,\theta_2^{s})$
                \end{algorithmic}
            \caption{Fast Gradient Matching along Trajectory ({\bf FGMT})}\label{algorithm:fast gradient matching along trajectory}
        \end{algorithm}
    \end{minipage}
    \vspace{-2.5mm}
\end{figure}

\subsection{Additional techniques}\label{section:additional techniques}

While Algorithm~\ref{algorithm:gradient matching along trajectory} solves the learning transfer problem~(eq.~\ref{gradient matching problem}) approximately, it still has some issues in terms of storage and computation cost. Here we explain two practical techniques to resolve them.

\paragraph{Linear trajectory.}

In terms of the storage cost, Algorithm~\ref{algorithm:gradient matching along trajectory} requires a capacity of $T+1$ times the model size to keep a learning trajectory of length $T$, which will be a more substantial issue as model size increases or the trajectory becomes fine-grained.
To reduce the required storage capacity, instead of keeping the entire trajectory, we propose to imitate it by linearly interpolating the end points.
In other words, given an initial parameter $\theta_1^0$ and the final $\theta_1^T$, we define a new trajectory $[\theta_1^0:\theta_1^T] := (\theta_1^{0}, \cdots, \theta_1^t,\cdots, \theta_1^{T})$ with $\theta_1^t := (1-\lambda_t)\theta_1^{0} + \lambda_t\theta_1^T$ and $0 = \lambda_0 \leq \cdots \leq \lambda_t \leq \cdots \leq \lambda_T = 1$.
\begin{wrapfigure}{tr}{0.325\textwidth}
  \vspace{-2mm}
      \includegraphics[width=\linewidth]{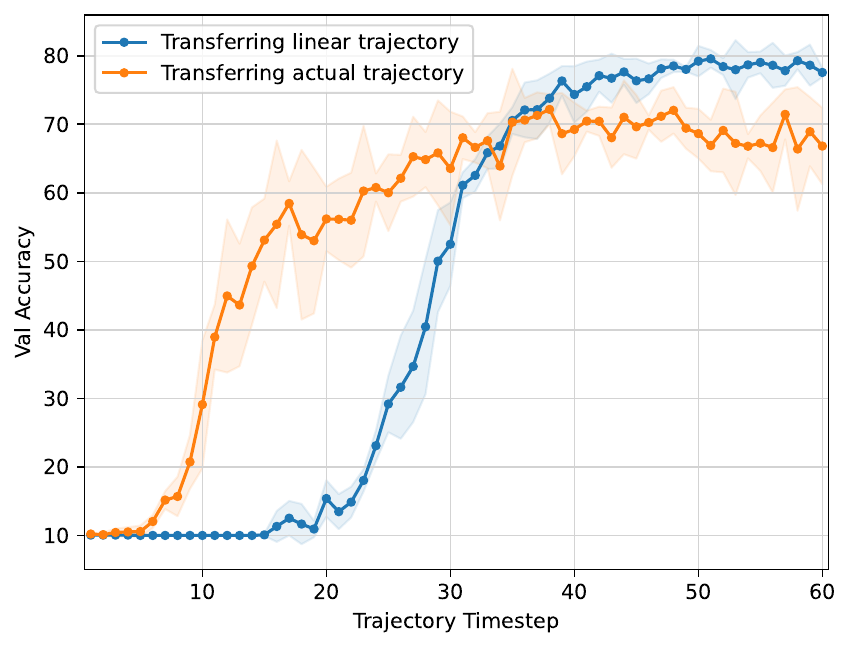}
      \vspace{-5mm}
  \caption{\small Linear vs. actual trajectory (Conv8 on CIFAR-10).}
  \label{figure:linear vs actual trajectory}
  \vspace{-3mm}
\end{wrapfigure}
Previous studies on the monotonic linear interpolation~\citep{goodfellow2015qualitatively,frankle2020revisiting} indicate that such a linearly interpolated trajectory satisfy our definition of learning trajectories in SGD training of modern deep neural networks.
Throughout this paper, we employ uniform scheduling for $\lambda_t$ defined by $\lambda_{t+1} - \lambda_t := 1/T$.
Next, we compare the transferred results between the linear trajectory $[\theta_1^0:\theta_1^T]$ and the actual trajectory $(\theta_\source^0,\cdots,\theta_\source^T)$ where each $\theta_\source^t$ is a checkpoint at the $t$-th training epoch on CIFAR-10 with $T=60$.
Interestingly, the transfer of the linear trajectory is more stable and has less variance than the transfer of the actual one.
This may be because the actual trajectory contains noisy information while the linear trajectory is  directed towards the optimal solution $\theta_\source^T$.
Due to its storage efficiency and stability in accuracy, we employ the linear trajectory of the length $T \leq 40$ throughout our experiments in Section~\ref{section:experiments}.

\vspace{-2mm}
\paragraph{Gradient caching.}

In terms of the computation cost, Algorithm~\ref{algorithm:gradient matching along trajectory} requires $O(T^2)$ times gradient computation.
To reduce the number of gradient computation, we propose to cache the gradients once computed instead of re-computing them for every $s=1,\cdots,T$.
In fact, the cached gradients $\nabla_{\theta_{2,\pi_s}^{t-1}} \mathcal{L}$ and the re-computed gradients $\nabla_{\theta_{2,\pi_{s'}}^{t-1}} \mathcal{L}$ are not the same quantity exactly since the intermediate parameter $\theta_{2,\pi_s}^{t-1}=\theta_2^{0}+\pi_s (\theta_1^{t-1} - \theta_1^{0})$ takes different values for each $s$. However, they can be treated as approximately equal by Lemma~\ref{lemma:consistency of subproblems} if we assume the continuity of the gradients.
Now we can reduce the number of gradient computation from $O(T^2)$ to $O(T)$ by caching the gradients once computed.
We describe this computationally efficient version in Algorithm~\ref{algorithm:fast gradient matching along trajectory}.

\begin{figure}
    \centering
    \begin{subfigure}[t]{0.325\textwidth}
        \centering
        \includegraphics[width=\textwidth]{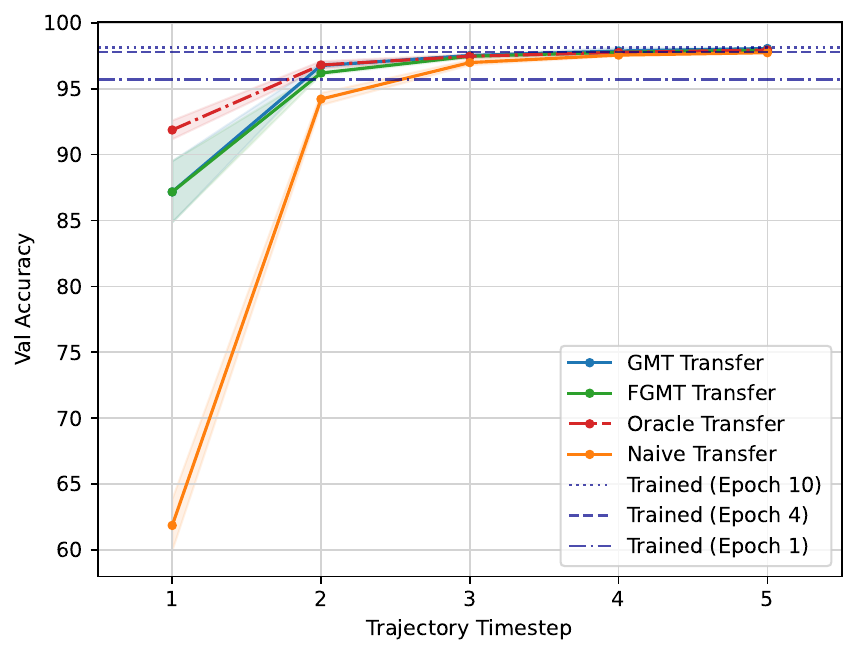}
        \caption{MNIST (2-MLP)}\label{figure: mnist with mlp}
    \end{subfigure}
    \hfill
    \begin{subfigure}[t]{0.325\textwidth}
        \centering
        \includegraphics[width=\textwidth]{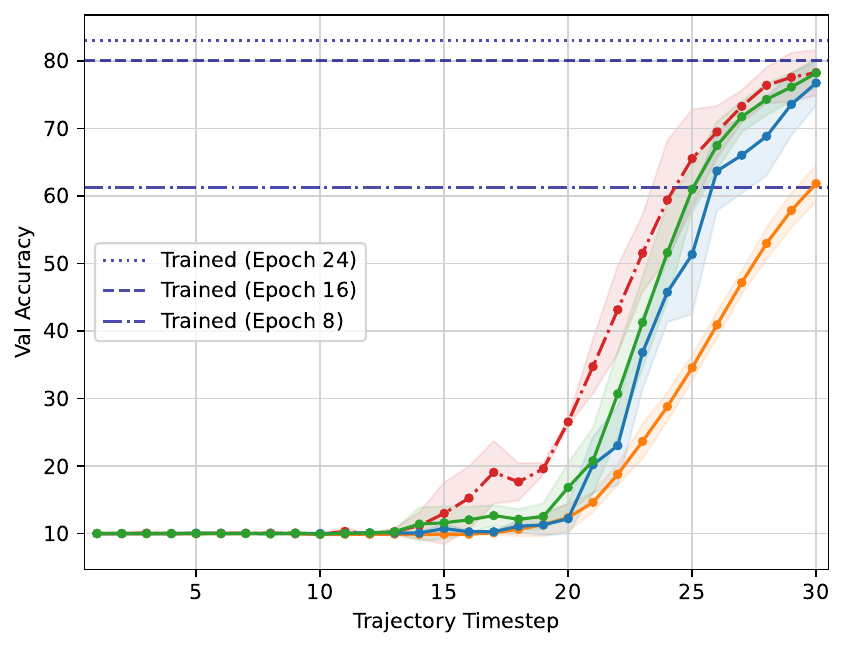}
        \caption{CIFAR-10 (Conv8)}\label{figure: cifar10 with conv8}
    \end{subfigure}
    \hfill
    \begin{subfigure}[t]{0.325\textwidth}
        \centering
        \includegraphics[width=\textwidth]{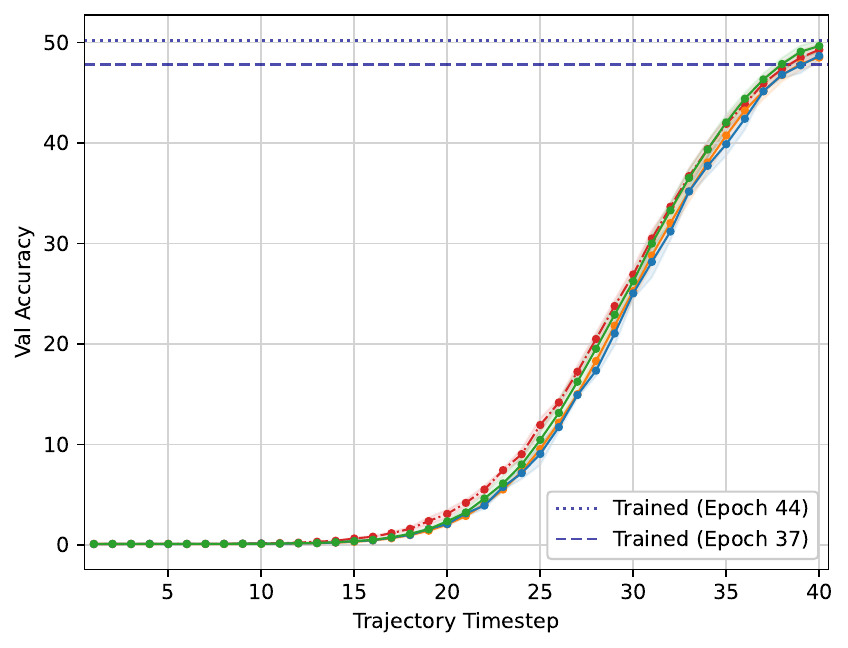}
        \caption{ImageNet (ResNet-18)}\label{figure: imagenet with resnet18}
    \end{subfigure}
    \vskip\baselineskip
    \begin{subfigure}[t]{0.325\textwidth}
        \centering
        \includegraphics[width=\textwidth]{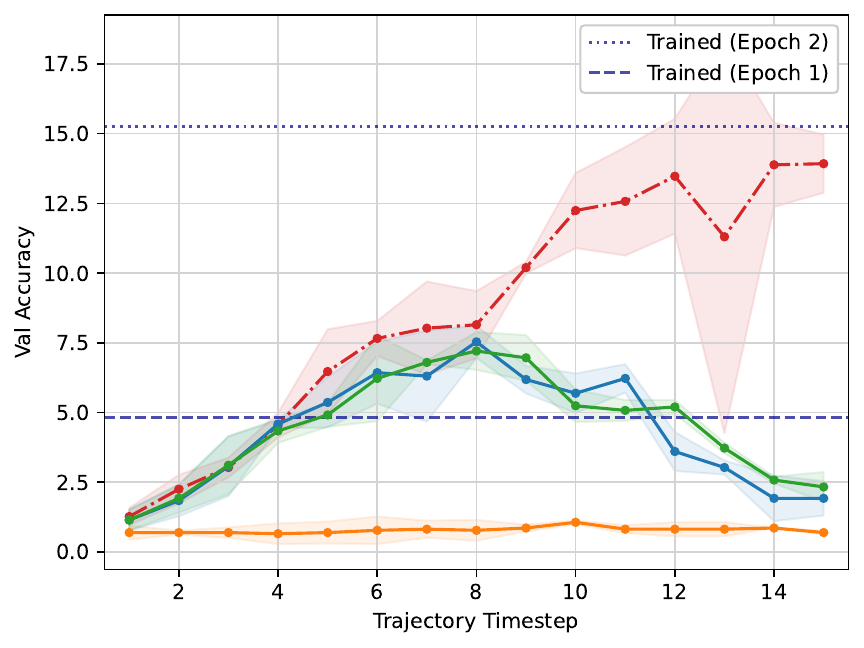}
        \caption{ImageNet $\to$ Cars}\label{figure: imagenet to cars}
    \end{subfigure}
    \begin{subfigure}[t]{0.325\textwidth}
        \centering
        \includegraphics[width=\textwidth]{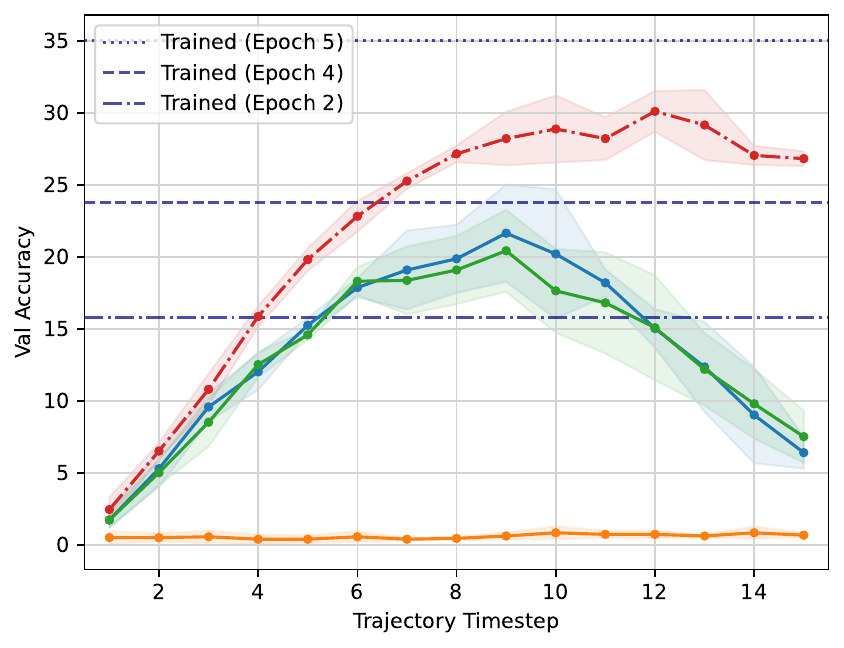}
        \caption{ImageNet $\to$ CUB}\label{figure: imagenet to cub}
    \end{subfigure}
    \begin{subfigure}[t]{0.325\textwidth}
      \centering
      \includegraphics[width=\textwidth]{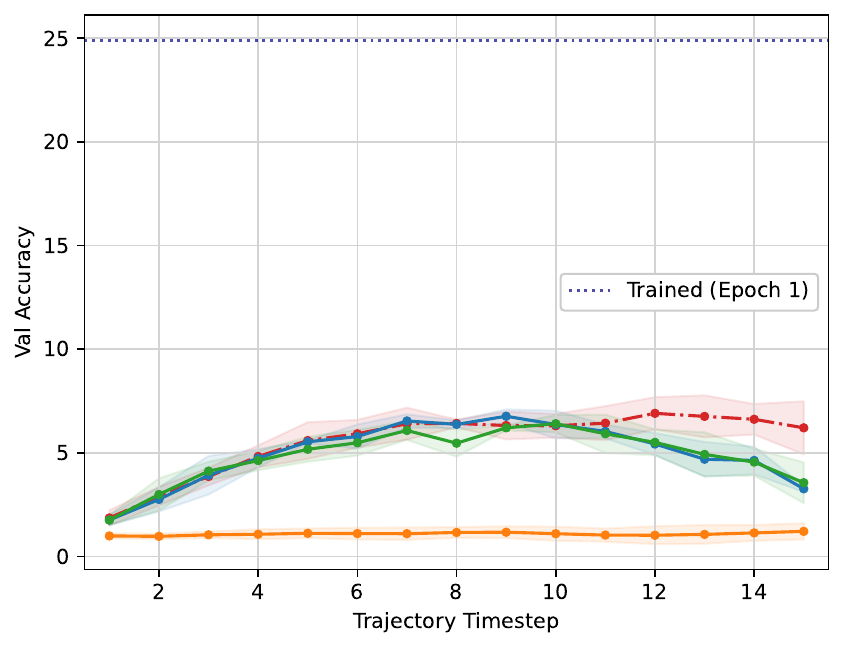}
      \caption{CIFAR-10 $\to$ CIFAR-100}\label{figure: cifar10 to cifar100}
    \end{subfigure}
  \caption{We plot the validation accuracies of the transferred parameter $\theta_{\transf,\pi_t}^t$ for each $t=1,\cdots,T$ with various datasets and NN architectures. We also provide the standard deviation over three runs for each experiment. The dotted bars show accuracies in standard training. (Upper) Transfer of a learning trajectory on a single dataset between random initial parameters. (Lower) Transfer of a fine-tuning trajectory between pre-trained parameters. For example, "ImageNet $\to$ Cars" means the transfer of the fine-tuning trajectory on the Cars dataset between the parameters pre-trained on ImageNet. }
    \label{figure:transfer between random and pretrained inits}
    \vspace{-1mm}
\end{figure}

\section{Experiments}\label{section:experiments}

In this section, we empirically evaluate how learning transfer works on standard vision datasets.
First, we compare our proposed methods ({\bf GMT}, {\bf FGMT}) and two baselines ({\bf Naive}, {\bf Oracle}), which are explained below, under the following two scenarios: (1) transferring learning trajectories starting from randomly initialized parameters and (2) transferring learning trajectories starting from pre-trained parameters (Section~\ref{section:experiments:learning transfer}).
Next, we evaluate how efficiently the transferred parameters can be trained in their subsequent training. (Section~\ref{section:experiments:acceleration}).
Finally, we investigate what properties the transferred parameters inherit from the target initializations from viewpoints of loss landscape and generalization (Section~\ref{section:what is being inherited}).
The details on experimental settings are provided in Appendix~\ref{appendix:details on experiments}.

\vspace{-3mm}
\paragraph{Baselines.}

As baselines for learning transfer, we introduce two natural methods: {\bf Naive} and {\bf Oracle}.
Both in the two baselines, we transfer a given learning trajectory $(\theta_1^0,\cdots,\theta_1^T)$ by a single permutation $\pi_{\naive}$ or $\pi_{\oracle}$, according to the problem formulation in Section~\ref{section:problem formulation}.
In the Naive baseline, we define $\pi_{\naive}$ as the identity permutation, which satisfies $\pi_{\naive} \theta = \theta$.
In other words, the transferred parameter by Naive is simply obtained as $\theta_{\transf,\pi_\naive}^t = \theta_\transf^0 + (\theta_\source^t - \theta_\source^0)$.
On the other hand, in the Oracle baseline, we first obtain a true parameter $\theta_\transf^T$ by actually training the given initial parameter $\theta_\transf^0$ with the same optimizer as training of $\theta_\source^T$.
Then we define $\pi_{\oracle}$ by minimizing the layer-wise $L^2$ distance between the actually trained trajectories $\theta_\transf^T-\theta_\transf^0$ and $\pi_\oracle(\theta_\source^T-\theta_\source^0)$, where we simply apply the coordinate descent as explained in Section~\ref{section:preliminaries:parameter alignment by permutation}.
The Oracle baseline is expected to be close to the optimal solution for the learning transfer problem via permutation symmetry.

\vspace{-3mm}
\paragraph{Source trajectories.}

In our experiments, as discussed in Section~\ref{section:additional techniques}, we consider to transfer linear trajectories $[\theta_\source^0:\theta_\source^T]$ of length $T$ rather than actual trajectories for $\theta_\source^T$ due to the storage cost and instability emerging from noise. The transferred results for actual trajectories instead of linear ones can be found in Appendix~\ref{appendix:section:learning transfer for actual trajectories}.

\subsection{Learning transfer experiments}\label{section:experiments:learning transfer}

Figure~\ref{figure:transfer between random and pretrained inits} shows the validation accuracies of the transferred parameters $\theta_{\transf,\pi_t}^t$ for each timestep $t=1,\cdots,T$ during the transfer.
For the baselines (Naive and Oracle), we set the $t$-th permutation $\pi_t$ by the fixed $\pi_\naive$ and $\pi_\oracle$ for every $t$.
For our algorithms (GMT and FGMT), the $t$-th permutation $\pi_t$ corresponds to $\pi_s$ in Algorithm~\ref{algorithm:gradient matching along trajectory} and \ref{algorithm:fast gradient matching along trajectory}.

In the upper figures \ref{figure: mnist with mlp}-\ref{figure: imagenet with resnet18}, we transfer a learning trajectory trained with a random initial parameter on a single dataset (MNIST~\citep{lecun2010mnist}, CIFAR-10~\citep{krizhevsky2009cifar10} and ImageNet~\citep{deng2009imagenet}) to another random initial parameter.
We will refer to this experimental setting as the random initialization scenario.
We can see that our methods successfully approximate the Oracle baseline.
Also, we can see that FGMT, the fast approximation version of GMT, performs very similarly to or even outperforms GMT.
This is probably because the update of $\pi_t$ affects the previously computed gradients in GMT, but not in FGMT, resulting in the stable behavior of FGMT.

In the lower figures \ref{figure: cifar10 to cifar100}-\ref{figure: imagenet to cub}, we transfer a learning trajectory of fine-tuning on a specialized dataset (a $10$-classes subset of CIFAR-100~\citep{krizhevsky2009cifar10}, Stanford Cars~\citep{krause20133d} and CUB-200-2011~\citep{wah2011caltech}) from an initial parameter that is pre-trained on ImageNet to another pre-trained one.
We refer to this experimental setting as the pre-trained initialization scenario.
This scenario seems to be more difficult to transfer the learning trajectories than the random initialization scenario shown in the upper figures, since the Naive baseline always fails to transfer the trajectories.
We can see that, while our methods behave closely to the Oracle baseline up to the middle of the timestep, the accuracy deteriorates immediately after that.
Nevertheless, the peak accuracies of our methods largely outperform those of the Naive baseline.
By stopping the transfer at the peak points (i.e., so-called early stopping), we can take an advantage of the transferred parameters as we will see in the next section.

\begin{figure}
  \centering
  \begin{subfigure}[t]{0.24\textwidth}
      \centering
      \includegraphics[width=\textwidth]{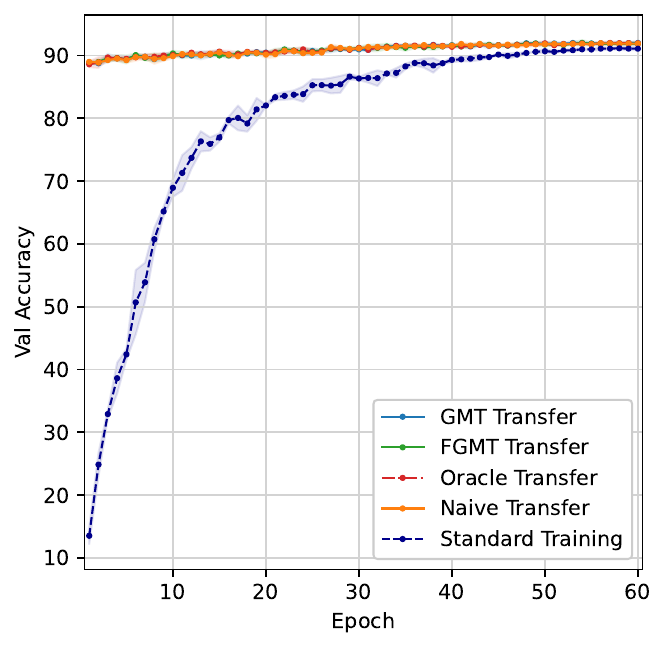}
      \caption{CIFAR-10 (Conv8)}\label{figure:fine-tuning on CIFAR-10 with transferred params}
  \end{subfigure}
  \hfill
  \begin{subfigure}[t]{0.24\textwidth}
      \centering
      \includegraphics[width=\textwidth]{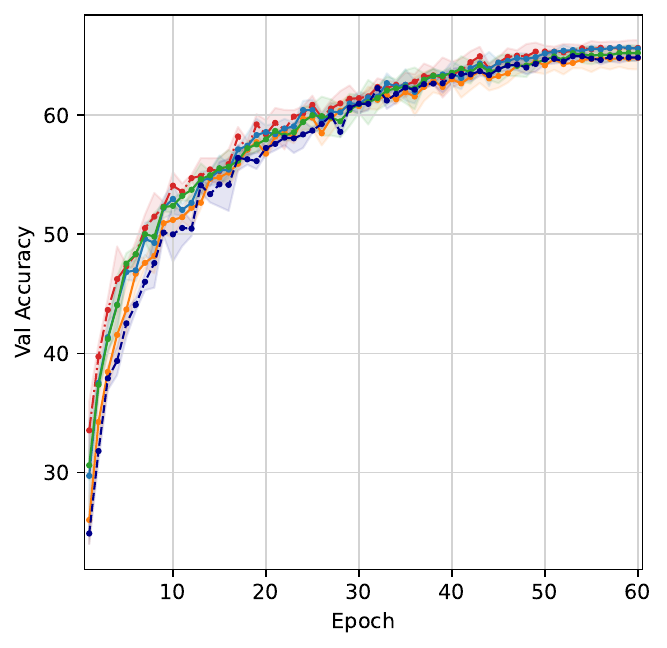}
      \caption{CIFAR10$\to$CIFAR100}\label{figure:fine-tuning on CIFAR-100 with transferred params}
  \end{subfigure}
  \hfill
  \begin{subfigure}[t]{0.24\textwidth}
      \centering
      \includegraphics[width=\textwidth]{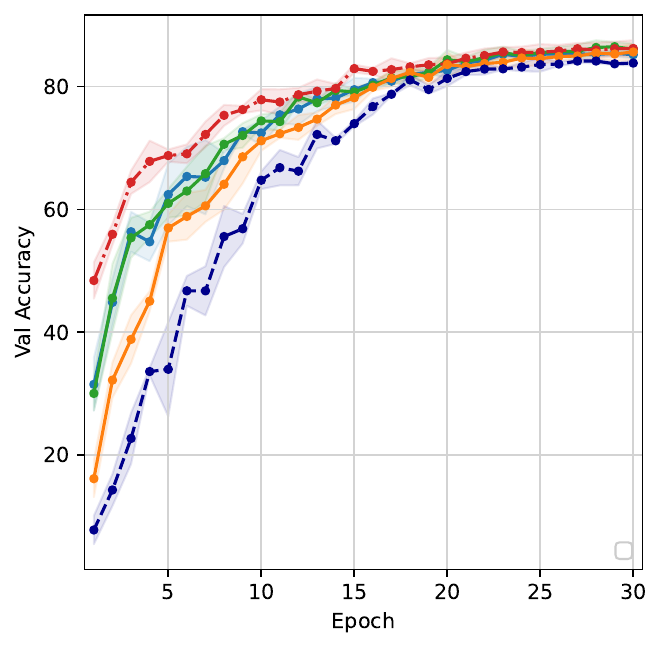}
      \caption{ImageNet $\to$ Cars}\label{figure:fine-tuning on Cars with transferred params}
  \end{subfigure}
  \hfill
  \begin{subfigure}[t]{0.24\textwidth}
      \centering
      \includegraphics[width=\textwidth]{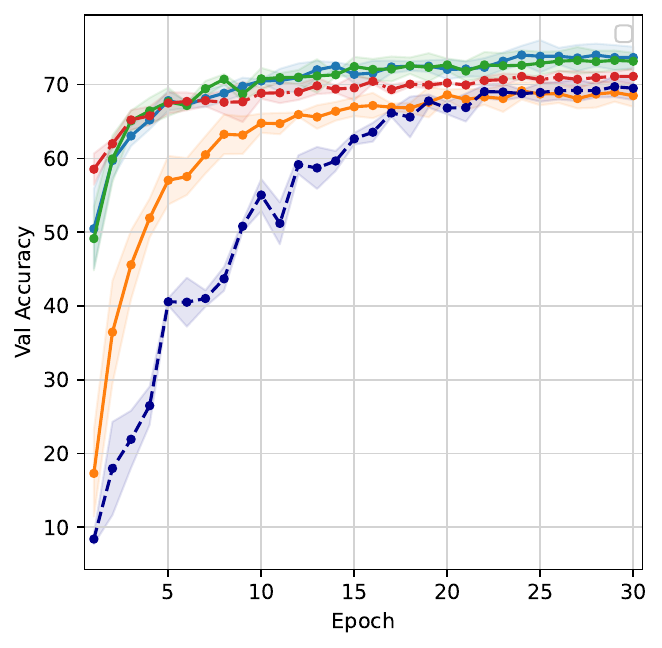}
      \caption{ImageNet $\to$ CUB}\label{figure:fine-tuning on CUB with transferred params}
  \end{subfigure}
  \vspace{-2mm}
  \caption{Subsequent training of the transferred parameters. In each figure, we plot the validation accuracies during the training of the transferred parameters, on the same dataset as the source trajectory being trained on. The transferred parameters obtained by solving the equation~(\ref{gradient matching problem}) can be trained faster than standard training, and Naive baseline in pre-trained initialization scenario. }\label{figure:fine-tuning of the transferred parameters}
  \vspace{-3mm}
\end{figure}

\subsection{Accelerated training of transferred parameters}\label{section:experiments:acceleration}

In the previous section, we obtained the transferred parameters that achieve non-trivial accuracy without any direct training.
Here we evaluate how efficiently the transferred parameters can be trained in their subsequent training, by training them on the same dataset for same epochs as the source trajectory.
We started each training from the transferred parameter $\theta_{\transf,\pi_t}^t$ at the best trajectory step $t$ in Figure~\ref{figure:transfer between random and pretrained inits}.
Figure~\ref{figure:fine-tuning of the transferred parameters} shows the validation accuracies for each epoch in the training of the transferred parameters.
In all cases, the transferred parameter can be trained faster than standard training from random/pre-trained initializations (denoted by \textbf{Standard Training}).
In the random initialization scenario (\ref{figure:fine-tuning on CIFAR-10 with transferred params}), there seem almost no difference between four transfer methods.
This is because the Naive baseline also achieves non-trivial accuracy already when transferred in this scenario.
On the other hand, in the pre-trained scenarios (\ref{figure:fine-tuning on CIFAR-100 with transferred params}), (\ref{figure:fine-tuning on Cars with transferred params}), (\ref{figure:fine-tuning on CUB with transferred params}), the parameters transferred by our methods and the Oracle baseline learns the datasets faster than the parameters transferred by the Naive baseline.
Thus the benefit of the transferred parameters seems to be greater especially in the pre-trained initialization scenario than in the random initialization scenario.

\subsection{What is being inherited from target initialization?}\label{section:what is being inherited}

In previous sections, we have observed that transferring a given learning trajectory to target initialization can achieve non-trivial accuracy beyond random guessing and indeed accelerate the subsequent training after transfer.
Then the following question arises: what does the transferred parameter differ from the source parameter and inherit from the target initialization? In particular, pre-trained initializations may have their own "character" making them different from each other, due to their generalization ability or prediction mechanism.

\paragraph{Basin in loss landscape.}

In context of transfer learning,
\citet{neyshabur2020being} empirically showed that fine-tuning one pre-trained initialization always leads to the same basin in loss landscape, which enables the fine-tuned models to share similar properties inherited from the same pre-trained initialization.
From this viewpoint, the parameter that is transferred to a target pre-trained initialization and then fine-tuned (referred as \textbf{FGMT+FT}) is expected to arrive at the same basin as the actually fine-tuned parameter (referred as \textbf{Target}) from the same pre-trained initialization.
In Figure~\ref{figure:inheritance of basin in loss landscape}, we empirically validate this expectation, with comparison to the source parameter permuted by Git Re-basin~\citep{ainsworth2023git} which is also designed to linearly mode connected to Target (referred as \textbf{Source (Permuted)}).
We can see that the transferred parameter lives in the nearly same basin as the target one, while there are still mild barriers from the permuted source parameter.
This implies that the transferred parameter inherits similar mechanism \citep{lubana2023mechanistic} from the target initialization  which cannot be obtained by simply permuting the source parameter.
\vspace{-3mm}

\begin{figure}
    \centering
    \begin{subfigure}[t]{0.32\textwidth}
      \centering
      \includegraphics[width=\textwidth]{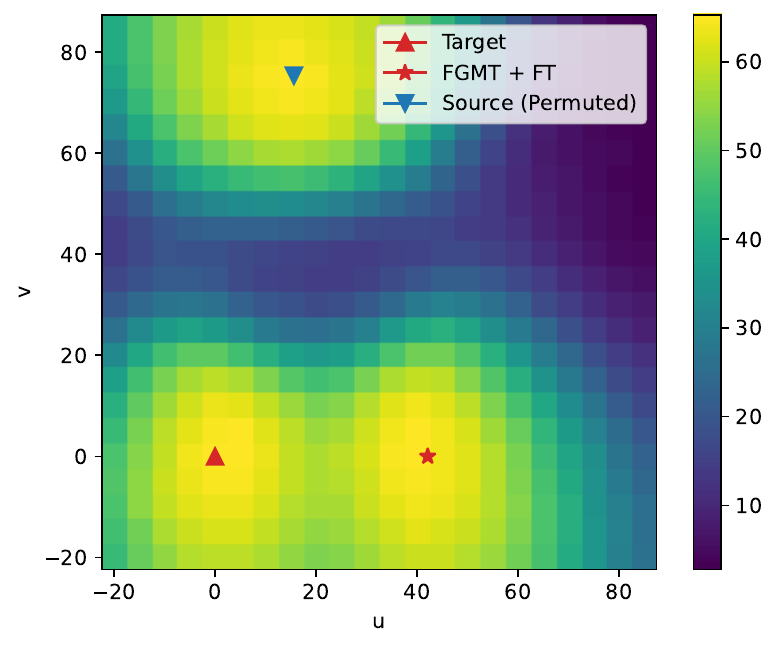}
      \vspace{-6mm}
      \caption{CIFAR10$\to$CIFAR100}
    \end{subfigure}
    \hfill
    \begin{subfigure}[t]{0.32\textwidth}
      \centering
      \includegraphics[width=\textwidth]{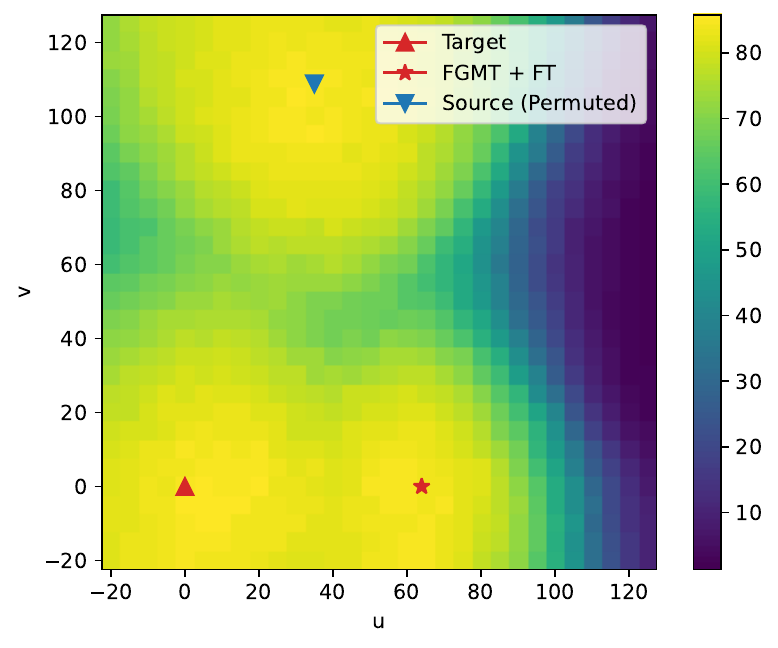}
      \vspace{-6mm}
      \caption{ImageNet$\to$Cars}
    \end{subfigure}
    \hfill
    \begin{subfigure}[t]{0.32\textwidth}
      \centering
      \includegraphics[width=\textwidth]{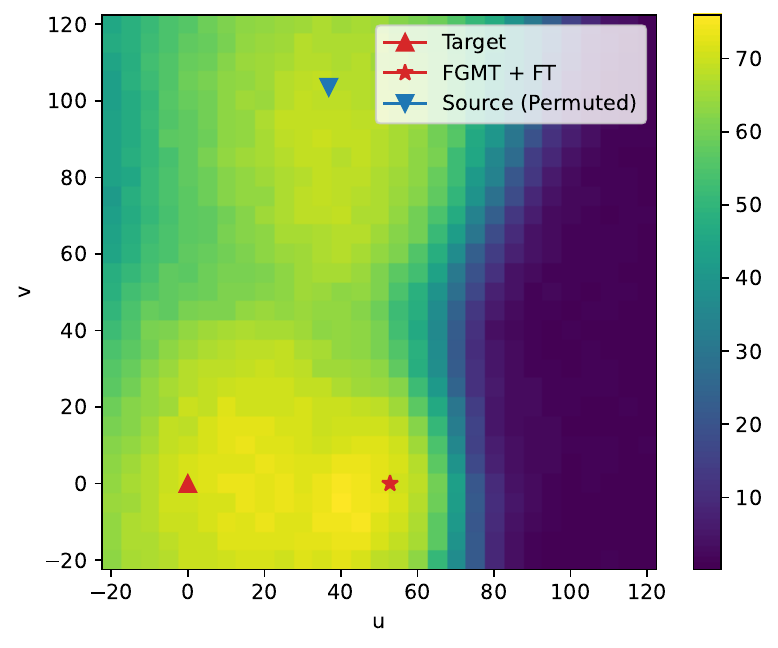}
      \vspace{-6mm}
      \caption{ImageNet$\to$CUB}
    \end{subfigure}
    \vspace{-2mm}
    \caption{Inheritance of basin in loss landscape. We plotted the validation accuracies over the $uv$-plane following the same protocol as \citet{garipov2018loss}. }\label{figure:inheritance of basin in loss landscape}
    \vspace{-2mm}
  \end{figure}
  
  \begin{figure}
    \centering
    \begin{subfigure}[t]{0.49\textwidth}
      \begin{subfigure}[t]{0.49\textwidth}
        \centering
        \includegraphics[width=\textwidth]{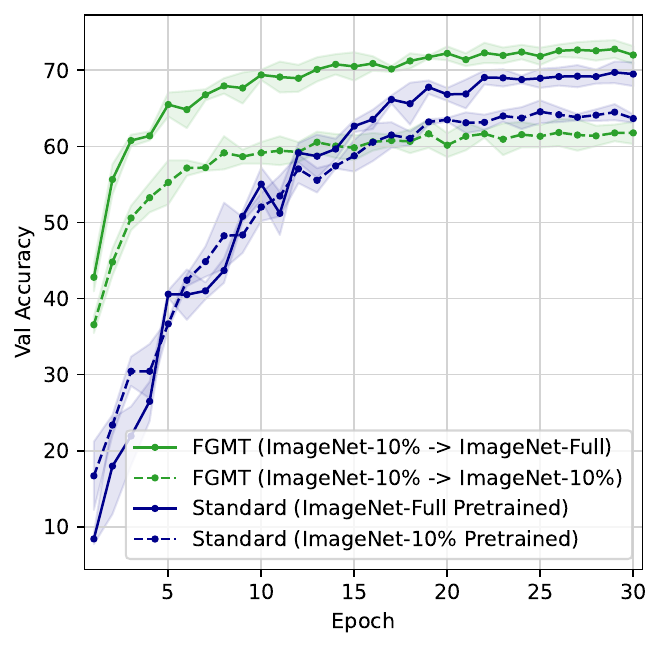}
      \end{subfigure}
      \hfill
      \begin{subfigure}[t]{0.49\textwidth}
          \centering
          \includegraphics[width=\textwidth]{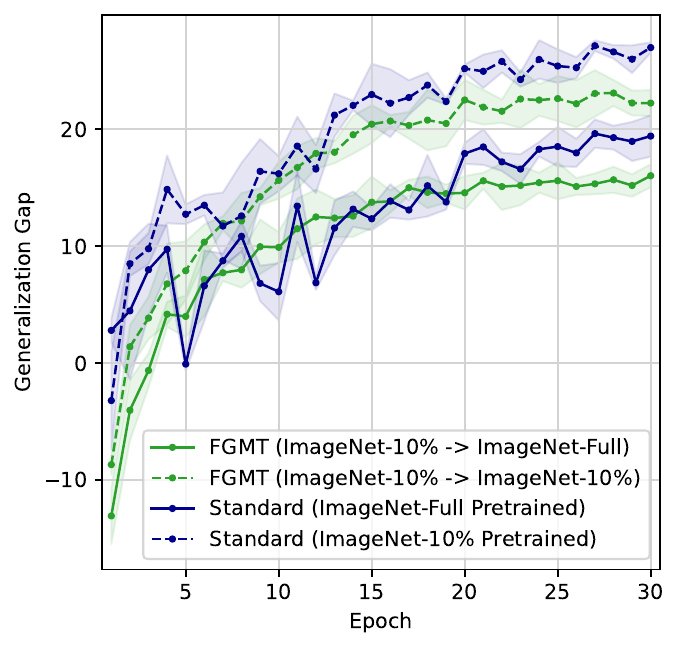}
        \end{subfigure}
        \vspace{-6mm}
        \caption{CUB}
    \end{subfigure}
    \hfill
    \begin{subfigure}[t]{0.49\textwidth}
      \centering
      \begin{subfigure}[t]{0.49\textwidth}
        \includegraphics[width=\textwidth]{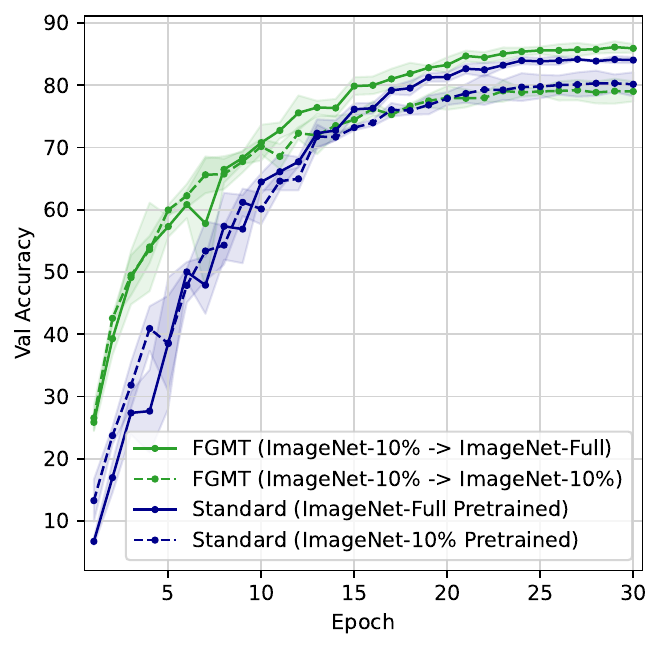}
      \end{subfigure}
      \hfill
      \begin{subfigure}[t]{0.49\textwidth}
        \includegraphics[width=\textwidth]{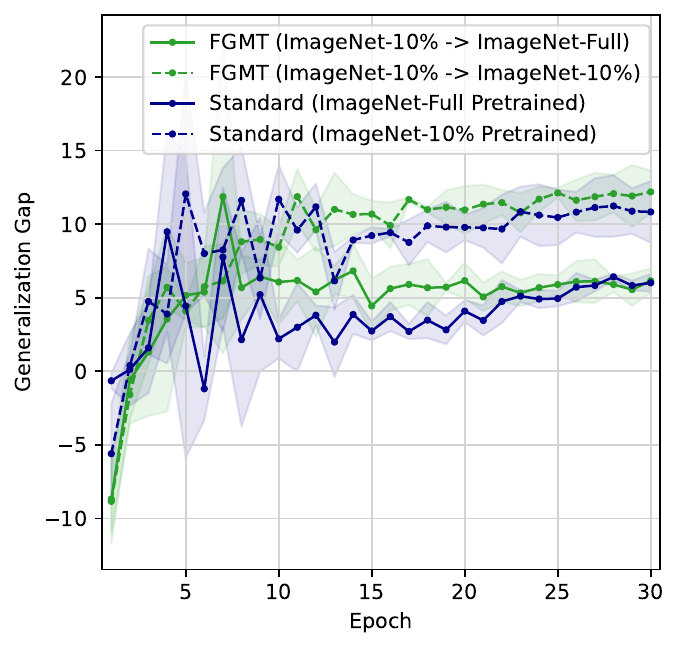}
      \end{subfigure}
      \vspace{-6mm}
      \caption{Cars}
    \end{subfigure}
    \vspace{-3mm}
    \caption{Inheritance of generalization ability. We plotted validation accuracies and generalization gaps between training and validation accuracies. }\label{figure:inheritance of generalization ability}
    \vspace{-5mm}
  \end{figure}

\paragraph{Generalization ability.}
In Figure~\ref{figure:inheritance of generalization ability}, unlike previous sections, we focus on the case that generalization ability of  source initialization $\theta_\source^0$ differs from target initialization $\theta_\transf^0$.
In particular, we use the source initialization $\theta_\source^0$ pre-trained with $10\%$ of training data from ImageNet and the target initialization $\theta_\transf^0$ pre-trained on full ImageNet.
The former (validation accuracy $\approx 50\%$) generalizes poorly than the latter (validation accuracy $\approx 72\%$).
Surprisingly, the results show that (1) learning trajectories can be transferred between pre-trained initializations with different generalization ability, and (2) the transferred parameter successfully inherits the better generalization ability from the target initialization $\theta_\transf^0$ even when the source trajectory itself is trained from the poorly generalizing initialization $\theta_\source^0$.
\vspace{-3mm}

\section{Conclusion}\label{section:conclusion}
\vspace{-2mm}

In this work, we formulated the problem of how we can synthesize an unknown learning trajectory from the known one, named the learning transfer problem, and derived an algorithm that approximately solves it very efficiently.
In our experiments, we confirmed that our algorithm efficiently transfers a given learning trajectory to achieve non-trivial accuracy without training, and that the transferred parameters accelerate their subsequent training.
Moreover, we investigated what properties the transferred parameters inherit from the target initializations from the viewpoints of loss landscape and generalization.
The last observation potentially opens up new paradigm of deep learning in future: For example, when a foundation model is newly updated with better generalization or fixed vulnerability, its fine-tuned models may efficiently follow it by transferring their fine-tuning trajectory from old foundation model to the new one.

%% file: appendix_body.tex
\renewcommand{\thefootnote}{\fnsymbol{footnote}}

\begin{appendix}

\section{Validating Assumption (P) at initialization}\label{appendix:validating assumption at initialization}

In this section, we theoretically validate Assumption (P) in Section~\ref{section:algorithm} for $2$-layered ReLU neural network at initialization.
Let $d$ be the dimension of inputs $x=(x_j)_{j=1,\cdots,d} \in \mathbb{R}^d$, and $f_{w,v}(x) := \sum_{i=1}^N v_i \sigma(\sum_{j=1}^d w_{ij} x_j)$ be a ReLU neural network\footnote{We can ignore the non-differentiable locus of the network because we only consider expected gradients over a distribution with a continuous density function, where such low-dimensional locus has zero measure. } with $N$ hidden neurons, where $w=(w_{ij})\in\mathbb{R}^{N\times d}, v=(v_i)\in\mathbb{R}^N$, $\sigma(z) := \max(z, 0)$.

We assume that the input $x\in\mathbb{R}^d$ and its label $y\in\mathbb{R}$ is sampled from some bounded distribution $\mathcal{D}$ whose density function $p_\mathcal{D}(x,y)$ is continuous over $(x,y) \in \mathbb{R}^{d+1}$. The parameters $w$ and $v$ are initialized with Kaiming uniform initialization, i.e.,  $w \sim U(-\frac{1}{\sqrt{d}}, \frac{1}{\sqrt{d}})^{N\times d}$, $v\sim U(-\frac{1}{\sqrt{N}}, \frac{1}{\sqrt{N}})^{N}$.
We employ the MSE loss $\mathcal{L}(y_1, y_2) := \frac{1}{2}|y_1 - y_2|^2$ and focus on the expected gradient $\mathbb{E}_{(x,y)\sim\mathcal{D}} [ \nabla_{(w, v)} \mathcal{L}(f_{w,v}(x), y) ]$ for this loss function at the initialization $w,v$.
Here $\nabla_{(w, v)} \mathcal{L}(f_{w,v}(x), y)$ lives in $\mathbb{R}^{N\times d}\times \mathbb{R}^{d}$ which is obtained as a concatenation of two vectors $(\partial\mathcal{L}(f_{w,v}(x), y) / \partial w_{ij} : i\in\{1,\cdots,N\},j\in\{1\,\cdots,d\}) \in \mathbb{R}^{N\times d}$ and $(\partial\mathcal{L}(f_{w,v}(x), y) / \partial v_{i} : i\in\{1,\cdots,N\} ) \in\mathbb{R}^{N}$.
Moreover, we introduce the following assumption on the data distribution $\mathcal{D}$ for technical reason:
\begin{equation}\label{equation: assumption on distribution}
  \Big|\mathbb{E}_{(x,y)\sim\mathcal{D}} \Big[ y \sigma\Big(\sum_{j} w_{ij} x_j\Big) \Big] \Big| \geq K \text{ for some } K > 0 \text{ with high probability w.r.t. } w_{ij},
\end{equation}
which is reasonable because it requires non-zero correlation between the input $x$ and its label $y$.

To validate our assumption, we introduce normalized distance $\frac{\lVert v_1 - v_2\rVert_2}{\sqrt{\lVert v_1\rVert_2 \lVert v_2\rVert_2}}$ between two vectors $v_1$ and $v_2$, which can be considered as cosine distance when $\lVert v_1\rVert \approx \lVert v_2 \rVert_2$ holds.
Now our claim is that the normalized distance of gradients at independent random initializations can be arbitrarily small by appropriate neuron-permutation and sufficient over-parameterization.
In this sense, the following theorem validates Assumption (P) at random initializations:
\begin{theorem}\label{appendix:theorem:validating assumption at initialization}
    Under the above assumption, given two pairs of randomly initialized parameters $(w, v)$ and $(w', v')$, with high probability, there exists a permutation symmetry $\pi \in S_N$ such that the normalized distance between the expected gradients $\mathbb{E}_{(x,y)} [\nabla_{w,v} \mathcal{L}]$ and $\mathbb{E}_{(x,y)} [\nabla_{w'',v''} \mathcal{L}]$, where $(w'', v'')$ is the permuted parameter of $(w', v')$ with $\pi$, can be arbitrarily small when $N$ is sufficiently large.
\end{theorem}
\begin{proof}
The expected gradient at $(w,v)$ can be computed as follows:

\begin{eqnarray*}
    \mathbb{E}_{(x,y)\sim\mathcal{D}}\left[ \frac{\partial \mathcal{L}(f_{w,v}(x), y)}{\partial w_{ij}} \right]
     &=& \mathbb{E}\left[ (f_{w,v}(x) - y) \frac{\partial f_{w,v}(x)}{\partial w_{ij}} \right] \\
     &=& v_i \mathbb{E}\left[ (f_{w,v}(x) - y) x_j \mathbbm{1}_{ \{\sum_{k}w_{ik}x_k > 0\} } \right], \\
    \mathbb{E}_{(x,y)\sim\mathcal{D}}\left[ \frac{\partial \mathcal{L}(f_{w,v}(x), y)}{\partial v_{i}} \right]
     &=& \mathbb{E}\left[ (f_{w,v}(x) - y) \frac{\partial f_{w,v}(x)}{\partial v_{i}} \right] \\
     &=& \mathbb{E}\left[ (f_{w,v}(x) - y) \sigma\left(\sum_{j=1}^d w_{ij}x_j\right) \right] \\
     &=& \mathbb{E}\left[ (f_{w,v}(x) - y) \left( \sum_{j=1}^d w_{ij}x_j \right)  \mathbbm{1}_{ \{\sum_{k}w_{ik}x_k > 0\} } \right] \\
     &=& \sum_{j=1}^d w_{ij} \mathbb{E}\left[ (f_{w,v}(x) - y) x_j  \mathbbm{1}_{ \{\sum_{k}w_{ik}x_k > 0\} } \right]
\end{eqnarray*}

By Hoeffding's inequality and $v_i \sim U(-\frac{1}{\sqrt{N}}, \frac{1}{\sqrt{N}})$, with high probability, we can assume $f_{w,v}(x) \approx 0$ for sufficiently large $N$. Thus by letting $C^j(w_{i1},\cdots,w_{id}) := \mathbb{E}\left[ - y x_j  \mathbbm{1}_{ \{\sum_{k}w_{ik}x_k > 0\} } \right]$, we can simplify the expected gradients as follows:

\begin{eqnarray*}
    \mathbb{E}_{(x,y)\sim\mathcal{D}}\left[ \frac{\partial \mathcal{L}(f_{w,v}(x), y)}{\partial w_{ij}} \right] &=& v_i C^j(w_{i1},\cdots,w_{id}), \\
    \mathbb{E}_{(x,y)\sim\mathcal{D}}\left[ \frac{\partial \mathcal{L}(f_{w,v}(x), y)}{\partial v_{i}} \right] &=& \sum_{j=1}^d w_{ij} C^j(w_{i1},\cdots,w_{id}),
\end{eqnarray*}

which are also valid for the counterpart $(w', v')$ instead of  $(w, v)$.

Next, we take a permutation $\pi \in S_N$ by applying Lemma~\ref{lemma: existence of permutation} to $(d+1)$-dimensional random vectors $(\sqrt{d}w_{i1},\cdots,\sqrt{d}w_{id},\sqrt{N}v_i)$ and $(\sqrt{d}w'_{i1},\cdots,\sqrt{d}w'_{id},\sqrt{N}v'_i) \sim U([-1,1]^{d+1})$ where $i=1,\cdots,N$.
Let $w'' := (w'_{\pi(i)j})_{ij}$ and $v'':=(v'_{\pi(i)})_{i}$.
Then we want to evaluate the squared normalized distance between $\mathbb{E} [ \nabla_{(w,v)} \mathcal{L} ]$ and $\mathbb{E} [ \nabla_{(w'',v'')} \mathcal{L} ]$:
\begin{eqnarray}\label{equation: normalized distance with permutation}
    \frac{ \left\lVert \mathbb{E} [ \nabla_{(w,v)} \mathcal{L} ] -  \mathbb{E} [ \nabla_{(w'',v'')} \mathcal{L} ] \right\rVert^2_2 }{ \left\lVert \mathbb{E} [ \nabla_{(w,v)} \mathcal{L} ]  \right\rVert_2 \left\lVert \mathbb{E} [ \nabla_{(w'',v'')} \mathcal{L} ]\right\rVert_2  }
\end{eqnarray}

We can evaluate each factor $\left\lVert \mathbb{E} [ \nabla_{(w,v)} \mathcal{L} ] -  \mathbb{E} [ \nabla_{(w'',v'')} \mathcal{L} ] \right\rVert^2_2$, $\left\lVert \mathbb{E} [ \nabla_{(w,v)} \mathcal{L} ] \right\rVert_2$, $\left\lVert \mathbb{E} [ \nabla_{(w',v')} \mathcal{L} ] \right\rVert_2$ in Eq~(\ref{equation: normalized distance with permutation}) as follows:

Claim (1): $ \left\lVert \mathbb{E} [ \nabla_{(w,v)} \mathcal{L} ] -  \mathbb{E} [ \nabla_{(w'',v'')} \mathcal{L} ] \right\rVert^2_2 = O(N\varepsilon^2) + o(N) $ with high probability.

Claim (2): $ \left\lVert \mathbb{E} [ \nabla_{(w,v)} \mathcal{L} ] \right\rVert_2^2  = \Omega(N), \ \left\lVert \mathbb{E} [ \nabla_{(w',v')} \mathcal{L} ] \right\rVert_2^2  = \Omega(N)$ with high probability.

\vspace{2mm}

\underline{Proof of Claim (1)}: 
Let $I := \{i\in\{1,\cdots,N\} : |w_{ij} - w''_{ij}| \leq \sqrt{d}\varepsilon, |v_{i} - v''_{i}| \leq \sqrt{N}\varepsilon \}$.  We have
\begin{eqnarray}
    \left\lVert \mathbb{E} [ \nabla_{(w,v)} \mathcal{L} ] -  \mathbb{E} [ \nabla_{(w'',v'')} \mathcal{L} ] \right\rVert^2_2  \nonumber
    &=& \sum_{i,j} \left| v_i C^j(w_{i*}) - v''_i C^j(w_{i*}'') \right|^2  \nonumber \\ 
    & & + \sum_i \Big| \sum_j w_{ij}C^j(w_{i*}) - w_{ij}''C^j(w''_{i*}) \Big|^2 \nonumber \\
    &\leq& \sum_{i=1}^N\sum_{j=1}^d \Big(\big|v_i - v''_i\big|\cdot\big|C^j(w_{i*})\big| + \big|v_i''\big|\cdot\big|C^j(w_{i*})-C^j(w''_{i*}))\big| \Big)^2 \nonumber \\
    & & + \sum_i \Big| \sum_j w_{ij}C^j(w_{i*}) - w_{ij}''C^j(w''_{i*}) \Big|^2 \nonumber \\
    &\leq& \sum_{i\in I}\sum_{j=1}^d \Big(\big|v_i - v''_i\big|\cdot\big|C^j(w_{i*})\big| + \big|v_i''\big|\cdot\big|C^j(w_{i*})-C^j(w''_{i*}))\big| \Big)^2 \nonumber \\
    & & + \sum_{i\in I} \Big| \sum_j w_{ij}C^j(w_{i*}) - w_{ij}''C^j(w''_{i*}) \Big|^2 \nonumber \\
    & & + O(|I^c|), \label{inequality: evaluation of normalized distance}
\end{eqnarray}
where $I^c$ stands for the complement $\{1,\cdots,N\}\setminus I$. By Lemma~\ref{lemma: existence of permutation}, the size of the complement $|I^c|$ is upper bounded by $O(\sqrt{N})$.
Also we note that $C^j$ is bounded by some constant since $\mathcal{D}$ is bounded, and $|v_j|, |v''_j| \leq 1/\sqrt{N}, |w_{ij}|, |w''_{ij}|\leq 1/\sqrt{d}$. Combining these facts and Lemma~\ref{lemma: evaluation of gradient at v_j} for the second term of (\ref{inequality: evaluation of normalized distance}), it follows that $(\ref{inequality: evaluation of normalized distance}) \leq O(N\varepsilon^2) + O(\sqrt{N})$.

\underline{Proof of Claim (2)}: 
\begin{eqnarray*}
  \left\lVert \mathbb{E}_{(x,y)\sim\mathcal{D}} [ \nabla_{(w,v)} \mathcal{L} ] \right\rVert^2_2 
  &\geq& \sum_{i=1}^N \Big| \mathbb{E} \Big[ \frac{\partial \mathcal{L}}{\partial v_i} \Big] \Big|^2 \\
  &\approx& \sum_{i=1}^N \Big| \mathbb{E}_{(x,y)\sim\mathcal{D}} \Big[ y \sigma\big(\sum_{j=1}^d w_{ij}x_j\big) \Big] \Big|^2 \\
  &\geq& KN.
\end{eqnarray*}
Here we used $f_{w,v}(x) \approx 0$ as explained above and the assumption Eq~(\ref{equation: assumption on distribution}) on distribution $\mathcal{D}$.

Finally, by combining Claim~(1) and (2), we obtain that the normalized distance (Eq~\ref{equation: normalized distance with permutation}) converges to $C_1\varepsilon^2$ when $N\to\infty$, with some constant $C_1>0$, which can be arbitrarily small by controlling $\varepsilon$.
\end{proof}

\begin{lemma}\label{lemma: existence of permutation}
Let $\varepsilon > 0$ be any small real number.
Let $z_1, \cdots, z_N, z'_1, \cdots, z'_N, \sim U([-K, K]^{k})$ be i.i.d. uniform random variables where $K > 0$ is a fixed constant.
When $N$ is sufficiently large, with high probability, there exists a permutation $\pi\in S^{N}$ such that the number of indices $i \in \{1,\cdots,N\}$ satisfying $|z_{ij}-z'_{\pi(i)j}| > \varepsilon$ for some $j$ is upper bounded by $O(\sqrt{N})$.
\end{lemma}
\begin{proof}
We can assume $K=1$ without loss of generality by re-scaling $z_i$, $z'_i$ and $\varepsilon$.
Here we follow the argument given in \citet{entezari2021role}.
For simplicity, we take $M\in\mathbb{N}$ satisfying $\frac{1}{M} \leq \varepsilon \leq \frac{1}{M-1}$.
For each $\mathbf{l} = (l_1,\cdots,l_k) \in L := \{1, \cdots, 2M\}^{k}$, we consider 
\begin{align*}
& Q_{\mathbf{l}} := (-1+\frac{l_1-1}{M},-1+\frac{l_1}{M})\times\cdots\times(-1+\frac{l_k-1}{M},-1+\frac{l_k}{M}), \\
& n_{\mathbf{l}} := \# \{ i\in I : z_i \in Q_{\mathbf{l}} \}, \ \ \ n'_{\mathbf{l}} := \# \{ i\in I : z'_i \in Q_{\mathbf{l}} \}.
\end{align*}
Note that any $z_i$ and $z'_i$ are contained in $Q_{\mathbf{l}}$ for some $\mathbf{l}\in L$ with probability $1$.
For each $\mathbf{l}\in L$, 
the number $n_{\mathbf{l}}$ and $n'_{\mathbf{l}}$ can be considered as the sum of random binary variables $b_1 + \cdots + b_N$ where each $b_j$ is sampled from Bernoulli distribution $\mathrm{Ber}(\frac{1}{M})$.
By Hoeffding's inequality for the Bernoulli random variables, we have
\begin{equation*}
  \mathbf{P}\left( \Big| n_{\mathbf{l}} - \frac{N}{M} \Big| \geq t  \right) \leq 2 \exp \left( -\frac{t^2}{2N} \right)
\end{equation*}
for each $\mathbf{l}\in L$. Thus, with probability $1-\delta$, we obtain
\begin{equation*}
\Big| n_{\mathbf{l}} - \frac{N}{M} \Big|, \Big| n'_{\mathbf{l}} - \frac{N}{M} \Big| \geq \sqrt{2N\log \left( \frac{4 |L|}{\delta} \right)}.
\end{equation*}

To construct the desired correspondence $\pi$ between $(z_i : 1\leq i \leq N)$ and $(z'_j: 1\leq j \leq N)$, each $z_i \in Q_\mathbf{l}$ should be mapped to some $z'_j \in  Q_\mathbf{l}$.
The number of $\{ i\in \{1, \cdots, N\} : z_i \text{ does not have its counterpart } z'_j \}$ can be upper bounded by
\begin{equation*}
  \sum_{\mathbf{l}\in L} \big|n_{\mathbf{l}} - n'_{\mathbf{l}} \big|
    \leq \sum_{\mathbf{l}\in L} \left|n_{\mathbf{l}} - \frac{N}{M}\right| + \left|n'_{\mathbf{l}} - \frac{N}{M}\right|
    \leq 2|L| \sqrt{2N\log \left( \frac{4 |L|}{\delta} \right)} = O(\sqrt{N}).
\end{equation*}
\end{proof}

\begin{lemma}\label{lemma: evaluation of gradient at v_j}
  Assume that $(w_1,\cdots,w_d), (w'_1,\cdots,w'_d) \in \big[-\frac{1}{\sqrt{d}}, \frac{1}{\sqrt{d}}\big]^d$ satisfy $\sup_k |w_k-w'_k|\leq \varepsilon$. It follows that
  $\big| \sum_j w_{j}C^j(w_1,\cdots, w_d) - \sum_j w'_j C^j(w'_1, \cdots, w'_d) \big| \leq O(\varepsilon)$.
\end{lemma}
\begin{proof}
Let $G(w_1,\cdots,w_d) := \sum_j w_{j}C^j(w_1,\cdots, w_d)=\mathbb{E}[-y\sigma(\sum_j w_j x_j)]$. By applying triangle inequality iteratively, we have
\begin{eqnarray*}
  & & |G(w_1,\cdots,w_d) - G(w'_1,\cdots,w'_d)| \\
  &\leq& |G(w_1,\cdots,w_d) - G(w'_1,w_2,\cdots,w_d)| + |G(w'_1,w_2,\cdots,w_d) - G(w'_1,\cdots,w'_d)| \\
  &\leq& \cdots \\
  &\leq& |G(w_1,\cdots,w_d) - G(w'_1,w_2,\cdots,w_d)| + \cdots + |G(w'_1,\cdots,w'_d,w_d) - G(w'_1,\cdots,w'_d)|
\end{eqnarray*}

Thus the proof can be reduced to the case where $|w_{k_0} - w'_{k_0}| \leq \varepsilon$ for some $k_0$ and $w_j=w'_j$ for $k\not=k_0$. We can assume $k_0=1$ and $w_1 \leq w'_1 \leq w_1 +  \varepsilon$ without loss of generality. Then we have
\begin{eqnarray}
  & & \big| G(w_1,\cdots,w_d) - G(w'_1,\cdots,w'_d) \big| \nonumber \\
 &=& \Big| \mathbb{E}\Big[ -y\sigma\Big(\sum_j w_{j} x_j\Big) \Big] - \mathbb{E}\Big[ -y\sigma\Big(\sum_j w'_{j} x_j\Big) \Big] \Big| \nonumber \\
 &=& \Big| \mathbb{E}\Big[ y \Big\{ \sigma\Big(\sum_j w_{j} x_j\Big) - \sigma\Big(\sum_j w'_{j} x_j\Big) \Big\} \Big] \Big| \nonumber \\
 &\leq& \int |y|\cdot \Big|\Big\{ \sigma\Big(w_1 x_1 + \sum_{j=2}^d w_{j} x_j\Big) - \sigma\Big(w'_1 x_1 + \sum_{j=2}^d w_{j} x_j\Big) \Big\}\Big| \cdot p(x,y)dxdy \nonumber \\
 &\leq& \int |y|\cdot \Big| w_1 x_1 + \sum_{j=2}^d w_{j} x_j \Big| \cdot \mathbbm{1}_{ \{ w_1x_1 + \sum_{j=2}^d w_jx_j \geq 0, w'_1x_1 + \sum_{j=2}^d w_jx_j < 0 \} }  \cdot p(x,y)dxdy \nonumber \\
 & & +  \int |y|\cdot \Big| w'_1 x_1 + \sum_{j=2}^d w_{j} x_j \Big| \cdot \mathbbm{1}_{ \{ w_1x_1 + \sum_{j=2}^d w_jx_j < 0, w'_1x_1 + \sum_{j=2}^d w_jx_j \geq 0 \} } \cdot p(x,y)dxdy \nonumber \\
 & & +  \int |y|\cdot \Big|  \Big(w_1 x_1 + \sum_{j=2}^d w_{j} x_j\Big) - \Big(w'_1 x_1 + \sum_{j=2}^d w_{j} x_j\Big)  \Big| \cdot p(x,y)dxdy \label{inequality: evaluation of gradient at v_j}
\end{eqnarray}
On the first term of the last inequality~(\ref{inequality: evaluation of gradient at v_j}), since $w'_1x_1 + \sum_{j=2}^d w_jx_j < 0$ holds on the integrated region, it follows that $0\leq w_1x_1 + \sum_{j=2}^d w_jx_j < w_1x_1 - w'_1x_1$. Thus we have 
\begin{eqnarray*}
  & & \int |y|\cdot \Big| w_1 x_1 + \sum_{j=2}^d w_{j} x_j \Big| \cdot \mathbbm{1}_{ \{ w_1x_1 + \sum_{j=2}^d w_jx_j \geq 0, w'_1x_1 + \sum_{j=2}^d w_jx_j < 0 \} }  \cdot p(x,y)dxdy \\
  &\leq&  \int |y|\cdot |w_1x_1 - w'_1x_1| \cdot p(x,y)dxdy \leq \int |y|\cdot \varepsilon |x_1| \cdot p(x,y)dxdy \\
\end{eqnarray*}
The same argument holds for the second term of inequality~(\ref{inequality: evaluation of gradient at v_j}). Furthermore, the third term is also bounded by $\int |y|\cdot \varepsilon |x_1| \cdot p(x,y)dxdy$.
Therefore, by the boundedness of the distribution $\mathcal{D}$, we have $(\ref{inequality: evaluation of gradient at v_j}) \leq 3\varepsilon \int |yx_1| p(x,y)dxdy = O(\varepsilon)$.
\end{proof}

\section{Proof of Lemma~\ref{lemma:consistency of subproblems}}

In this section, we employ the same notation as Section~\ref{section:algorithm}.
We assume that
\begin{enumerate}
\item $\lVert (\theta_i^{t+1} - \theta_i^{t}) - (-\alpha_t\nabla_{\theta_i^t} \mathcal{L}) \rVert_2  < \varepsilon$,
\item $\lVert \pi_s \nabla_{\theta_1^t}\mathcal{L} - \nabla_{\theta^t_{2,\pi_{s-1}}} \mathcal{L} \rVert_2 < \varepsilon$, for $t\leq s-1$,
\item The gradient $\nabla_\theta\mathcal{L}$ is $K$-Lipschitz continuous with respect to the parameter $\theta$, i.e., $\lVert \nabla_\theta\mathcal{L}-\nabla_{\theta'}\mathcal{L} \rVert_2 \leq K \lVert \theta - \theta'\rVert_2$.
\end{enumerate}

\begin{lemma}
Under the above assumptions, we have
\begin{equation}
    \theta^{t}_{2,\pi_{s'}} - \theta^t_{2,\pi_{s}}  = O(T^{s'} K^{s'} \varepsilon),
\end{equation}
for $0\leq t \leq s < s' \leq T$.
\end{lemma}
\begin{proof}
We prove by induction:
\begin{align*}
\theta^{t}_{2,\pi_{s'}} &= \theta_2^0 + \pi_{s'} (\theta_1^t - \theta_1^0) \\
                        &= \theta^0_2 + \pi_{s'} \left( -\sum_{t'=0}^{t-1} \alpha_{t'}\nabla_{\theta_1^{t'}} \mathcal{L} + O(t \varepsilon) \right)  & \text{(by Assumption 1.)}  \\
                        &= \theta^0_2 + \sum_{t'=0}^{t-1}  \left(- \alpha_{t'} \pi_{s'} \nabla_{\theta_1^{t'}}\mathcal{L} \right) + O(t\varepsilon) \\
                        &= \theta^0_2 + \sum_{t'=0}^{t-1} \left(- \alpha_{t'} \nabla_{\theta^{t'}_{2, \pi_{s'-1}}}\mathcal{L} + O(\varepsilon) \right) + O(t\varepsilon) & \text{(by Assumption 2.)} \\
                        &= \theta_2^0 + \sum_{t'=0}^{t-1} \left(- \alpha_{t'}  \nabla_{\theta^{t'}_{2, \pi_{s-1}}+O(T^{s'-1}K^{s'-1}\varepsilon)}\mathcal{L} + O(\varepsilon) \right) + O(t\varepsilon) & \text{(by induction hypothesis.)}\\
                        &= \theta_2^0 + \sum_{t'=0}^{t-1} \left(- \alpha_{t'} \nabla_{\theta^{t'}_{2, \pi_{s-1}}}\mathcal{L} \right) +O(T^{s'}K^{s'}\varepsilon) & \text{(by Assumption 3.)}\\
                        &= \theta_2^0 + \pi_{s}\left( - \sum_{t'=0}^{t-1} \alpha_{t'} \nabla_{\theta^{t'}_1}\mathcal{L} \right) +O(T^{s'}K^{s'}\varepsilon) \\
                        &= \theta_2^0 + \pi_{s}\left( \theta_1^{t} - \theta_1^0 \right) +O(T^{s'}K^{s'}\varepsilon) \\
                        &= \theta^t_{2,\pi_s} +O(T^{s'}K^{s'}\varepsilon).
\end{align*}

\end{proof}

\section{Related Work}\label{appendix:related works}

\paragraph{Loss landscape, linear mode connectivity, permutation symmetry.}

Loss landscape of training deep neural network has been actively studied in an effort to unravel mysteries of non-convex optimization in deep learning~\citep{hochreiter1997flat,choromanska2015loss,lee2016gradient,keskar2017onlargebatch,li2018visualizing}.
One of the mysteries in deep learning is the stability and consistency of their training processes and solutions, despite of the multiple sources of randomness such as random initialization, data ordering and data augmentation~\citep{fort2019deep,bhojanapalli2021reproducibility,summers2021nondeterminism,jordan2023calibrated}.
Previous studies of mode connectivity both theoretically~\citep{freeman2017topology,simsek2021geometry} and empirically~\citep{draxler2018essentially,garipov2018loss} demonstrate the existence of low-loss curves between any two optimal solutions trained independently with different randomness.

Linear mode connectivity (LMC) is a special case of mode connectivity where two optimal solutions are connected by a low-loss linear path~\citep{nagarajan2019uniform,frankle2020linear,mirzadeh2021linear,entezari2021role,benzing2022random,ainsworth2023git,juneja2023linear,lubana2023mechanistic}.
In this line of research, \citet{entezari2021role} observed that even two solutions trained from different random initialization can be linearly connected by an appropriate permutation symmetry.
\citet{ainsworth2023git} developed an efficient method to find such permutations, and \citet{jordan2023repair} extends it to NN architectures with Batch normalization~\citep{ioffe2015batch}.
Their observations strength the expectation on some sort of similarity between two training processes even from different random initializations, via permutation symmetry.
In our work, based on these observations, we attempt to transfer one training process to another initial parameter by permutation symmetry.

Another line of research related to our work is the studies of monotonic linear interpolation (MLI) between an initialization and its trained result.
\citet{goodfellow2015qualitatively} first observed that the losses are monotonically decreasing along the linear path between an initial parameter and the trained one.
\citet{frankle2020revisiting} and \citet{lucas2021onmonotonic} confirmed that the losses are monotonically non-increasing even with modern network architectures such as CNNs and ResNets~\citep{he2016deep}.
\citet{vlaar2022can} empirically analyzed which factor in NN training influences the shape of the non-increasing loss curve along the linear interpolation, and \citet{wang2023plateau} theoretically analyzed the plateau phenomenon in the early phase of the linear interpolation.
Motivated by these observations, we introduced the notion of linear trajectories in Section~\ref{section:additional techniques} to reduce storage costs in our learning transfer.

\vspace{-1mm}
\paragraph{Model editing.}

Our approach of transferring learning trajectories can be also considered as a kind of model editing~\citep{sinitsin2020editable,santurkar2021editing,ilharco2022patching,ilharco2023editing} in the parameter space because we modify a given initial parameter by adding an appropriately permuted trajectory.
In particular, a recent work by \citet{ilharco2023editing} is closely related to our work.
They proposed to arithmetically edit a pre-trained NN with a task vector, which is defined by subtracting the initial pre-trained parameter from the parameter fine-tuned on a specific task.
From our viewpoint, task vectors can be seen as one-step learning trajectories (i.e., learning trajectories with $T=1$).
Model merging (or model fusion)~\citep{singh2020model,matena2022merging,wortsman2022model,li2023branchtrainmerge} is also related in the sense of the calculation in the parameter space.

\vspace{-1mm}
\paragraph{Efficient training for multiple NNs.}
There are several literatures that attempt to reduce the computation costs in training multiple NNs.
Fast ensemble is an approach to reduce the cost in ensemble training by cyclically scheduled learning rate~\citep{huang2017snapshot} or by searching different optimal basins in loss landscape~\citep{garipov2018loss,fort2019deep,wortsman2021learning,benton2021loss}.
A recent work by \citet{liu2022knowledge} leverages knowledge distillation~\citep{hinton2015distilling} from one training to accelerate the subsequent trainings.
Our approach differs from theirs in that we try to establish a general principle to transfer learning trajectories.
Also, the warm-starting technique investigated by \citet{ash2020warm} seems to be related in that they subsequently train from a once trained network.
There may be some connection between their and our approaches, which remains for future work.

\vspace{-1mm}
\paragraph{Gradient matching.}

The gradient information obtained during training has been utilized in other areas outside of ours.
For example, in dataset distillation, \citet{zhao2021dataset} optimized a distilled dataset by minimizing layer-wise cosine similarities between gradients on the distilled dataset and the real one, starting from random initial parameters, which leads to similar training results on those datasets.
Similarly, \citet{yin2021see} successfully recovered private training data from its gradient by minimizing the distance between gradients.
In contrast to their problem where input data is optimized, our problem requires optimizing unknown transformation for NN parameters.
In addition, our problem requires matching the entire learning trajectories, which are too computationally expensive to be computed naively.

\section{Details for our experiments}\label{appendix:details on experiments}

\subsection{Datasets}

The datasets used in our experiments (Section~\ref{section:experiments}) are listed below. For all datasets, we split the officially given training dataset into 9:1 for training and validation.

\begin{itemize}
    \item {\bf MNIST.} \ \ MNIST~\citep{lecun2010mnist} is a dataset of $28\times 28$ images of hand-written digits, which is available under the  terms of the CC BY-SA 3.0 license.
    \item {\bf CIFAR-10, CIFAR100.} \ \  CIFAR-10 and CIFAR-100~\citep{krizhevsky2009cifar10} are datasets of $32\times 32$ images with $10$ and $100$ classes respectively.
    \item {\bf ImageNet.} \ \ ImageNet~\citep{deng2009imagenet} is a large-scale dataset of images with $1000$ classes, which is provided for non-commercial research or educational use.
    \item {\bf Stanford Cars.} \ \ Stanford Cars~\citep{krause20133d} is a dataset of images with 196 classes of cars, which is provided for
    research purposes. We refer to this dataset as Cars for short.
    \item {\bf CUB-200-2011.}  \ \ CUB-200-2011~\citep{wah2011caltech} is a dataset of images of 200 species of birds. We refer to this dataset as CUB for shot.
\end{itemize}

\subsection{Network architectures}

The neural network architectures used in our experiments (Section~\ref{section:experiments}) are listed as follows:

\begin{itemize}
    \item {\bf 2-MLP.} \ \ 2-MLP is a two-layered neural network with the ReLU activations. The design of this architecture is shown in Table~\ref{app:table:architecture of 2-MLP}.
    \item {\bf Conv8.} \ \ Conv8 is an $8$-layered CNN followed by three linear and ReLU layers. The design of this architecture is shown in Table~\ref{app:table:architecture of Conv8}.
    \item {\bf ResNet-18.} \ \ The ResNet family~\citep{he2016deep} is a series of deep CNNs with skip connections. We employed the standard $18$-layered one for ResNet-18.
\end{itemize}

\begin{table}[h]
    \caption{The architecture of 2-MLP.}
    \label{app:table:architecture of 2-MLP}
    \centering
    \begin{tabular}{cll}
        \toprule
        No. & Layers
            &  Output dimensions \\
        \midrule
        1 & Flattening & $784 \ (=28\times 28)$
        \\
        2 & Linear $\to$ ReLU
        & $4096$
        \\
        4 & Linear
        & $10$
        \\
        5 & Softmax
        & $10$
        \\
        \bottomrule
    \end{tabular}
\end{table}

\begin{table}[h]
\caption{The architecture of Conv8.}
\label{app:table:architecture of Conv8}
\centering
\begin{tabular}{cl}
    \toprule
    No. & Layers \\
    \midrule
    1 & Conv(input=$3$, output=$64$, kernel\_size=$(3,3)$, stride=$1$, padding=$1$) $\to$ ReLU \\
    2 & Conv(input=$64$, output=$64$, kernel\_size=$(3,3)$, stride=$1$, padding=$1$) $\to$ ReLU \\
    3 & MaxPooling(kernel\_size=(2, 2)) \\
    4 & Conv(input=$64$, output=$128$, kernel\_size=$(3,3)$, stride=$1$, padding=$1$) $\to$ ReLU \\
    5 & Conv(input=$128$, output=$128$, kernel\_size=$(3,3)$, stride=$1$, padding=$1$) $\to$ ReLU \\
    6 & MaxPooling(kernel\_size=(2, 2)) \\
    7 & Conv(input=$128$, output=$256$, kernel\_size=$(3,3)$, stride=$1$, padding=$1$) $\to$ ReLU \\
    8 & Conv(input=$256$, output=$256$, kernel\_size=$(3,3)$, stride=$1$, padding=$1$) $\to$ ReLU \\
    9 & MaxPooling(kernel\_size=(2, 2)) \\
    10 & Conv(input=$256$, output=$512$, kernel\_size=$(3,3)$, stride=$1$, padding=$1$) $\to$ ReLU \\
    11 & Conv(input=$512$, output=$512$, kernel\_size=$(3,3)$, stride=$1$, padding=$1$) $\to$ ReLU \\
    12 & MaxPooling(kernel\_size=(4, 4)) \\
    13 & Conv(input=$512$, output=$512$, kernel\_size=$(3,3)$, stride=$1$, padding=$1$) $\to$ ReLU \\
    14 & Linear(input=$512$, output=$256$) $\to$ ReLU \\
    15 & Linear(input=$256$, output=$256$) $\to$ ReLU \\
    16 & Linear(input=$512$, output=$10$) \\
    17 & Softmax
    \\
    \bottomrule
\end{tabular}
\end{table}

\subsection{Training details}

\subsubsection{Details on implementation and devices}

We implemented the codebase for all experiments in Python 3 with the PyTorch library~\citep{paszke2019pytorch}.
Our computing environment is a machine with $12$ Intel CPUs, $140$ GB CPU memory and a single A100 GPU.

\subsubsection{Training of source trajectories}
In training for source trajectories, we used SGD with momentum in PyTorch for the optimization.
It has the following hyperparameters: the total epoch number $E$, batch size $B$, learning rate $\alpha$, weight decay $\lambda$, momentum coefficient $\mu$.
We used the cosine annealing~\citep{loshchilov2017sgdr} for scheduling the learning rate $\eta$ except for MNIST.
For the random initialization, we used the standard Kaiming initialization~\citep{he2015delving}, which is also a default in PyTorch.
The details on the hyperparameters are as follows:
\begin{itemize}
    \item {\bf 2-MLP on MNIST.} \ \  We used $E=15$, $B=128$, $\alpha=0.01$, $\lambda=0.0$, $\mu=0.9$.
    \item {\bf Conv8 on CIFAR-10.} \ \ We used $E=60$, $B=128$, $\alpha=0.05$, $\lambda=0.0001$, $\mu=0.9$.
    \item {\bf Conv8 on CIFAR-100.} \ \ We used $E=30$, $B=128$, $\alpha=0.05$, $\lambda=0.0001$, $\mu=0.9$, starting from the pre-trained parameter on CIFAR-10.
    \item {\bf ResNet18 on ImageNet.} \ \  We used $E=100$, $B=128$, $\alpha=0.1$, $\lambda=0.0001$, $\mu=0.9$. For the first $5$ epochs, we gradually increased the learning rate as $\eta = 0.1 \times (i/5)$ for each $i$-th epoch ($i=1,\cdots,5$). For the last $95$ epochs, we decayed the learning rate by cosine annealing starting from $\eta = 0.1$.
    \item {\bf ResNet18 on Cars.} \ \  We used $E=30$, $B=128$, $\alpha=0.1$, $\lambda=0.0001$, $\mu=0.9$, starting from the pre-trained parameter on ImageNet.
    \item {\bf ResNet18 on CUB.} \ \ We used $E=30$, $B=128$, $\alpha=0.1$, $\lambda=0.0001$, $\mu=0.9$, starting from the pre-trained parameter on ImageNet.
\end{itemize}

\subsubsection{Hyperparameters for transferring learning trajectories}

Our methods (Algorithm~\ref{algorithm:gradient matching along trajectory},~\ref{algorithm:fast gradient matching along trajectory}) have the following hyperparameters: the length $T$ of learning trajectories, the batch size $B$ for each gradient matching. Also, for NN architectures with the Batch normalization~\citep{ioffe2015batch} such as ResNets, the batch size $B'$ for resetting the means and variances in the Batch Normalization layers~\citep{jordan2023repair} is also a hyperparameter.

\begin{itemize}
    \item {\bf 2-MLP on MNIST.} \ \  We used $B=128$ and $T=5$.
    \item {\bf Conv8 on CIFAR-10.} \ \ We used $B=128\times 2$ and $T=30$.
    \item {\bf Conv8 on CIFAR-100.} \ \ We used $B=128\times 2$ and $T=15$.
    \item {\bf ResNet18 on ImageNet.} \ \  We used $B=128\times 100$, $B'=128 \times 20$ and $T=40$.
    \item {\bf ResNet18 on Cars.} \ \  We used $B=128\times 5$, $B'=128\times 2$ and $T=15$.
    \item {\bf ResNet18 on CUB.} \ \ We used $B=128\times 5$, $B'=128\times 2$ and $T=15$.
\end{itemize}

\subsubsection{Computational cost of GMT/FGMT}\label{appendix: computational cost}

Total computational cost of FGMT (resp. GMT) consists of:  computing $2BT$ (resp. $2BT^2$) gradients, which is the same budget as $2T$ iterations of standard training, and solving lightweight (depending on model architecture, typically $1$ or $2$ seconds in our environment) optimizations (eq.~\ref{eq:linear sub-problems}) for $T$ times.
Thus, for FGMT, the required computational cost should be significantly smaller than standard training (if implemented optimally).

\subsubsection{Subsequent training of transferred parameters}
For the subsequent training in Section~\ref{section:experiments:acceleration}, we used the same optimizer and learning rate scheduler as training of source trajectories, with slightly small initial learning rates ($\alpha = 0.01$ on CIFAR-10, $\alpha = 0.05$ on CIFAR-100, $\alpha = 0.05$ on Cars, $\alpha = 0.01$ for CUB), which are selected based on the validation accuracy of Naive baseline for fair comparison.

\newpage

\section{Additional experimental results}

\subsection{Learning transfer for actual trajectories}\label{appendix:section:learning transfer for actual trajectories}

\begin{figure}[H]
  \centering
  \begin{subfigure}[t]{0.325\textwidth}
      \centering
      \includegraphics[width=\textwidth]{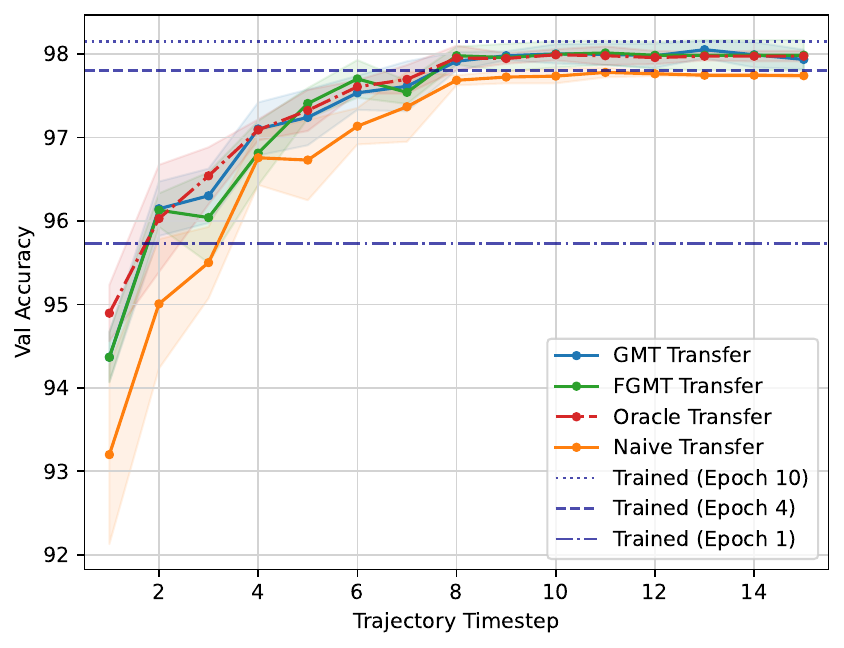}
      \caption{MNIST (2-MLP)}
  \end{subfigure}
  \hfill
  \begin{subfigure}[t]{0.325\textwidth}
      \centering
      \includegraphics[width=\textwidth]{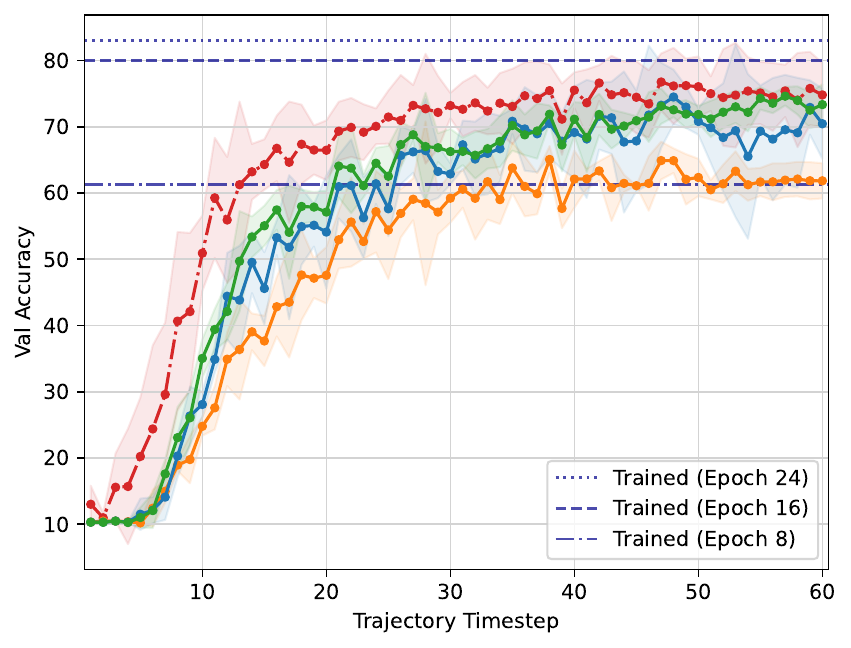}
      \caption{CIFAR-10 (Conv8)}
  \end{subfigure}
  \hfill
  \begin{subfigure}[t]{0.325\textwidth}
      \centering
      \includegraphics[width=\textwidth]{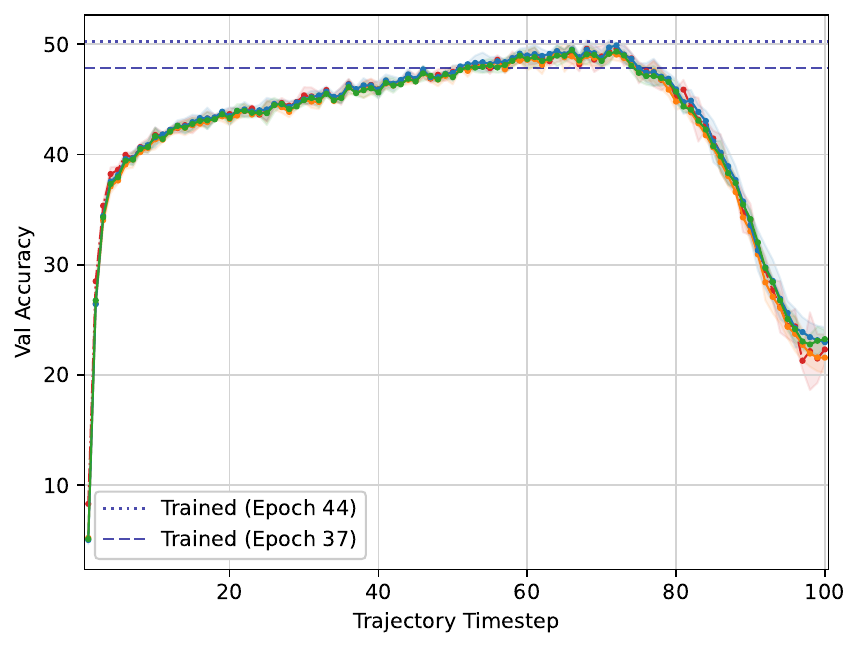}
      \caption{ImageNet (ResNet-18)}
  \end{subfigure}
  \vskip\baselineskip
  \begin{subfigure}[t]{0.325\textwidth}
      \centering
      \includegraphics[width=\textwidth]{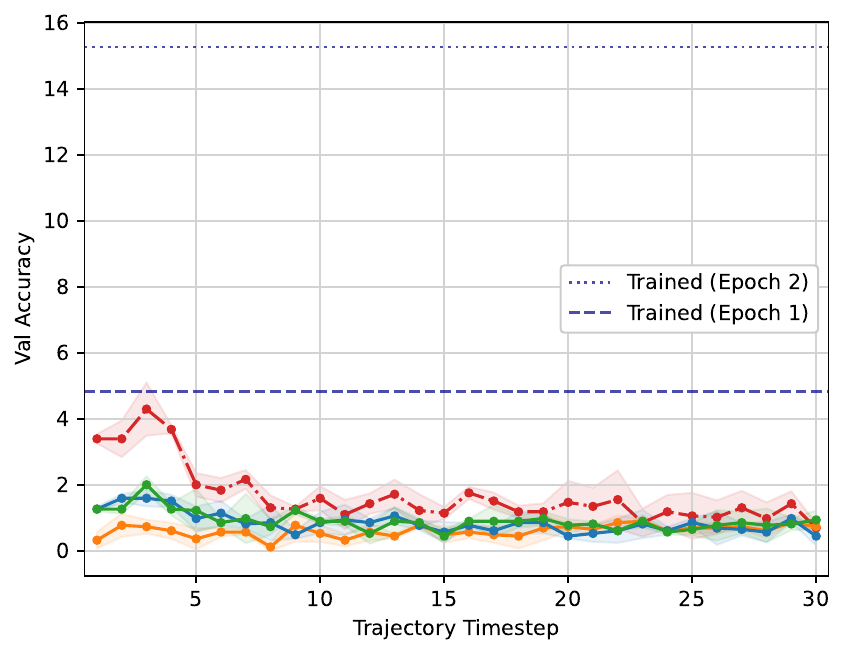}
      \caption{ImageNet $\to$ Cars}
  \end{subfigure}
  \begin{subfigure}[t]{0.325\textwidth}
      \centering
      \includegraphics[width=\textwidth]{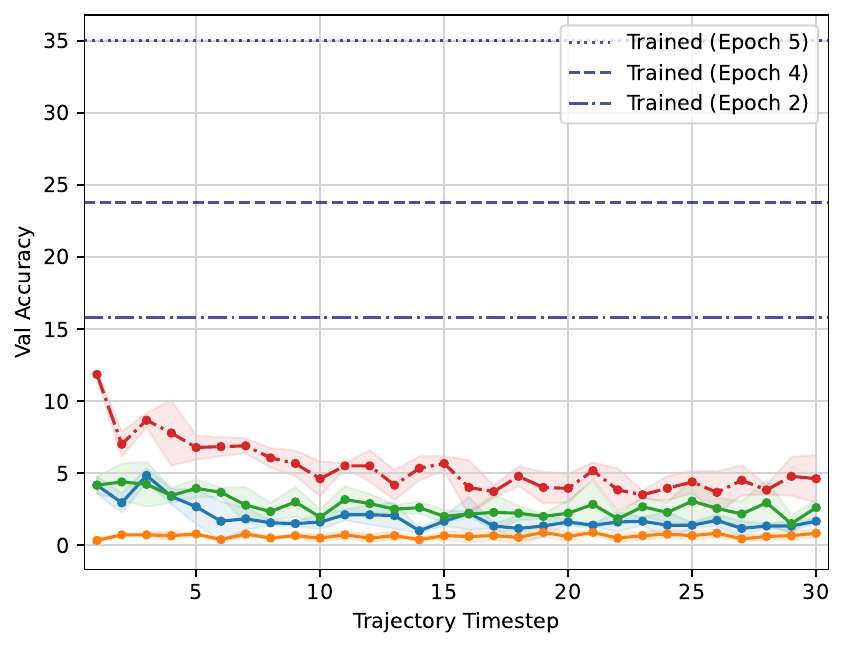}
      \caption{ImageNet $\to$ CUB}
  \end{subfigure}
  \begin{subfigure}[t]{0.325\textwidth}
    \centering
    \includegraphics[width=\textwidth]{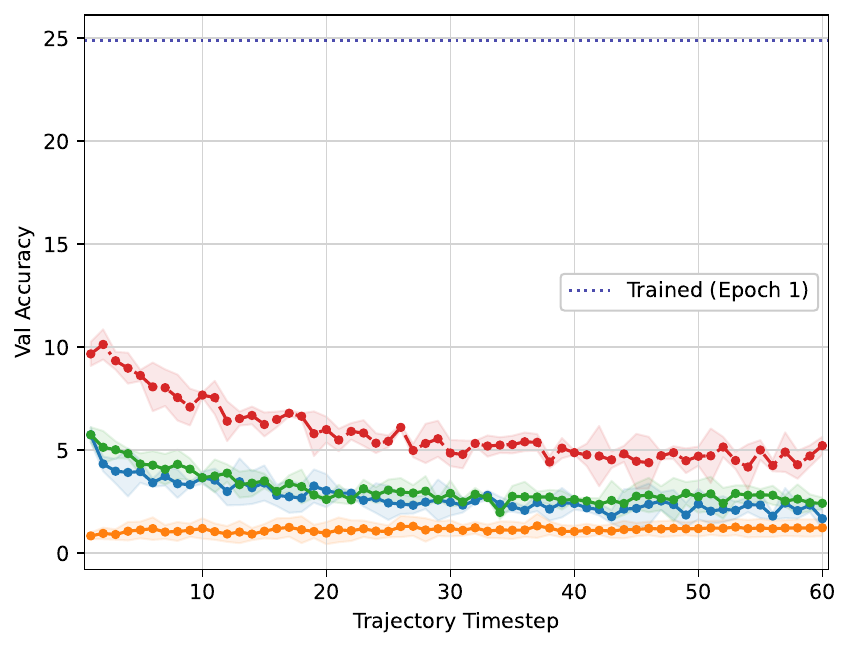}
    \caption{CIFAR-10 $\to$ CIFAR-100}
  \end{subfigure}
\caption{We plot the validation accuracies of the transferred parameter $\theta_{\transf,\pi_t}^t$ for each timestep $t=1,\cdots,T$ with various datasets and NN architectures as in Figure~\ref{figure:transfer between random and pretrained inits}, except that the actual trajectory $\theta_\source^t$ at $t$-th epoch of SGD is transferred instead of linear trajectories. Compared to the results for linear trajectories in Figure~\ref{figure:transfer between random and pretrained inits}, the transferred results for actual trajectories tend to have more variance in validation accuracy and fail to transfer in fine-tuning scenario. }
  \vspace{-1mm}
\end{figure}

\subsection{Ensemble evaluation}

\begin{table}[H]
  \centering
  \resizebox{0.7\columnwidth}{!}{
  \begin{tabular}{lccc}
    \toprule
      & GMT
      & FGMT
      & Full Ensemble
      \\
    \midrule
    CIFAR-10
    & $92.01$ (@50ep)
    & $92.11$ (@50ep)
    & $\mathbf{92.26}$ (@60ep)
    \\
    CIFAR-100
    & $\mathbf{70.41}$ (@50ep)
    & $69.53 $ (@50ep)
    & $69.81 $ (@60ep)
    \\
    Cars 
    & $86.83 $ (@20ep)
    & $86.61 $ (@20ep)
    & $\mathbf{87.10} $ (@30ep)
    \\
    CUB 
    & $74.71 $ (@10ep)
    & $\mathbf{75.62}$ (@10ep)
    & $74.02$ (@30ep)
    \\
    \bottomrule
  \end{tabular}
  }
  \caption{We evaluate ensembles of the transferred models with fewer subsequent training epochs than training from scratch. "@ X ep" means that each member of the ensembles are trained for X epochs after transferred (for GMT/FGMT) or from scratch (for Full Ensemble). }
  \label{table:common hyperparameters}
\end{table}

\end{appendix}